\documentclass{article} 
\usepackage{iclr2027_conference,times}

\usepackage{amsmath,amsfonts,bm}

\def\eqref#1{equation~\ref{#1}}

\def\1{\bm{1}}

\def\rmR{{\mathbf{R}}}

\DeclareMathAlphabet{\mathsfit}{\encodingdefault}{\sfdefault}{m}{sl}
\SetMathAlphabet{\mathsfit}{bold}{\encodingdefault}{\sfdefault}{bx}{n}

\def\gE{{\mathcal{E}}}

\usepackage{hyperref}
\usepackage{url}
\usepackage{amsthm}
\usepackage{algorithm}
\newtheorem{theorem}{Theorem}
\newtheorem{proposition}{Proposition}

\usepackage{booktabs}
\usepackage{wrapfig}
\usepackage{algpseudocode}
\usepackage{wrapfig}
\usepackage{caption}

\usepackage{siunitx}

\usepackage[table]{xcolor}

\definecolor{scorebest}{HTML}{EAF4CC}
\definecolor{scoregood}{HTML}{DDF1D8}
\definecolor{scoremid}{HTML}{D9ECDF}
\definecolor{scorelow}{HTML}{D6E3EF}
\definecolor{scorebad}{HTML}{DDD2EE}

\newcommand{\ours}{H-Glow}
\newcommand{\hazard}{\tau}
\newcommand{\jump}{\pi}

\usepackage{xcolor}

\usepackage{graphicx}
\usepackage{subcaption}
\usepackage{enumitem}

\usepackage{algorithm}
\usepackage{algpseudocode}
\usepackage{xcolor}
\usepackage{hyperref}
\usepackage{comment}
\usepackage[capitalize]{cleveref}
\crefname{equation}{Eq.}{Eqs.}
\Crefname{equation}{Eq.}{Eqs.}
\usepackage{xkcdcolors}
\usepackage{svg}
\hypersetup{
    colorlinks=true,
    linkcolor=xkcdDarkRose,
    urlcolor=xkcdBluish,
    linktoc=all,
    citecolor=xkcdFlatBlue,
}

\definecolor{ourscolor}{RGB}{232,240,252}
\definecolor{ourslight}{RGB}{240,245,252}
\definecolor{oursmain}{RGB}{220,234,250}

\algrenewcommand{\algorithmiccomment}[1]{%
  \hfill\(\triangleright\)\textcolor{gray}{ #1}%
}

\title{H-Glow: Accelerating Graph Generation via Hazard-Guided Discrete Flow Matching}

\iclrfinalcopy 

\author{%
  Sirine Ayadi, Stephan Günnemann \\
  School of Computation, Information and Technology,  Technical University of Munich \\
  Munich Data Science Institute \&
  Munich Center for Machine Learning \\
  \texttt{\{si.ayadi,s.guennemann\}@tum.de} 
}

\newcommand{\Cat}{\mathrm{Cat}}

\usepackage{natbib}
\begin{document}

\maketitle
\lhead{Preprint}
\begin{abstract}
Discrete flow matching (DFM) offers a promising approach to graph generation; however, generating high-quality graphs with few steps remains challenging due to the strong dependencies between nodes and edges. Although DFM decouples training from sampling, enabling flexible sampling strategies without retraining, existing graph DFM samplers often rely on task-specific heuristics, substantially expanding the hyperparameter space. Moreover, graph DFM models still require sampling budgets comparable to those of diffusion-based approaches, leaving their potential for accelerating graph generation only partially realized.
In this work, we introduce \ours{}, a hazard-guided framework for accelerating graph generation built on the decomposition of Continuous-Time Markov chain (CTMC) generators into hazards and conditional jump distributions. \ours{} leverages these quantities to plan graph edits and dynamically adjust the sampling time to graph activity. \ours{} can be adapted to pretrained graph DFM models, substantially improving the generation quality with fewer steps.
We further introduce a novel graph DFM model based on the conditional generator matching framework, which directly predicts hazards and conditional jumps, and show that its training objective decomposes exactly into hazard matching and hazard-weighted jump prediction. Experiments on molecular and synthetic graph benchmarks demonstrate improved generation quality with substantially fewer sampling steps. 
\end{abstract}
\section{Introduction}
Graph generative models have the potential to accelerate materials design \citep{sanchez2018inverse} and drug discovery \citep{zhang2025artificial} by designing new molecular structures.
Discrete-state graph diffusion models \citep{vignac2022digress} offer a promising approach for progressively refining categorical node and edge attributes through iterative denoising updates. However, their denoising process is defined on a fixed temporal grid, and sampling is therefore tied to the discrete time points used during training.

Continuous-time graph diffusion models overcome this restriction by learning to denoise at arbitrary time points, allowing flexible time discretization during inference \citep{campbell2022continuous,xu2024discrete,siraudin2024cometh}.
More recently, discrete flow matching (DFM) has achieved state-of-the-art performance on molecular and synthetic graph benchmarks \citep{qin2025defog}. Despite the sampling acceleration achieved by flow-based models in domains such as images \citep{geng2026mean} and text generation \citep{sevriugov2026logit}, graph flow models remain computationally costly and often require as many sampling steps as diffusion-based methods, leaving their potential for efficient generation only partially realized.

The structural complexity of graphs poses distinct challenges for efficient sampling compared with image or text generation, as node and edge attributes are strongly interdependent. For example, a single incorrect bond can violate valency constraints and invalidate a molecule.
Existing samplers update multiple nodes and edges in parallel without accounting for potential conflicts between updates, leading to violations of structural or chemical constraints. 
The approximation errors become more pronounced as the sampling budget decreases. To address this, prior work exploits the training-sampling decoupling of DFM and employs specialized sampling-time interventions to improve generation quality, including target guidance, additional stochasticity, and tailored sampling-time discretizations \citep{campbell2024generative,qin2025defog,luosimgfm}. However, these interventions expand the hyperparameter space and often require extensive tuning for each task and sampling budget during generation.

We present \textit{\ours{}}, a novel framework for graph generation built on the decomposition of CTMC generators into hazards and conditional jump distributions, as presented in \cref{fig:hglow-overview}.
\ours{} decomposes the generative process into three key components: \textbf{(i)} {hazard-guided planning} that prioritizes graph edits and coordinates node and edge updates; \textbf{(ii)} \emph{conditional jump distributions} that guide the refinement of selected graph components; and \textbf{(iii)} {a hazard-guided clock} that dynamically adjusts the sampling-time distortion to the graph’s predicted activity.
Our \ours{} sampler substantially accelerates graph generation and serves as a drop-in replacement for pretrained models with accessible CTMC generators. We further introduce the \ours{} model, a structured DFM model for graph generation trained using the generator-matching framework \citep{holderrieth2025generator}. \ours{} adopts a factorized CTMC parameterization to learn node and edge hazards, as well as conditional jump distributions \citep{li2026neural}, thereby directly supervising the quantities that guide the sampling.

We summarize our contributions as follows:
\vspace{-0.2em}
\begin{itemize}[leftmargin=1.8em,topsep=0pt, partopsep=0pt,itemsep=0pt, parsep=1pt]
    \item \textbf{Hazard-guided graph sampling.}
We introduce \ours{}, a novel sampler that effectively leverages the model's predicted hazards and jump distributions to prioritize graph edits and dynamically adapt sampling-time distortions to graph activity. Moreover, \ours{} serves as a drop-in sampler that accelerates pretrained models with accessible CTMC generators, without retraining, auxiliary networks, or complex sampling interventions.
    \item \textbf{Structured graph generator learning framework.}
    We introduce a graph DFM model that directly learns node and edge hazards, as well as conditional jump distributions, rather than relying on a denoiser network. Under the CTMC factorization, we derive an exact decomposition of the generator-matching objective into hazard matching and hazard-weighted jump prediction.
    \item \textbf{Empirical evaluation.}
    \ours{} achieves high validity and strong FCD performance on molecular benchmarks using substantially fewer sampling steps. Our ablations further demonstrate that the hazard-guided clock outperforms prior time discretizations and that our sampler surpasses existing samplers with fewer function evaluations.
\end{itemize}

\begin{figure}[t]
    \centering
    \setlength{\abovecaptionskip}{2pt}
    \setlength{\belowcaptionskip}{0pt}
    \includegraphics[width=\textwidth]{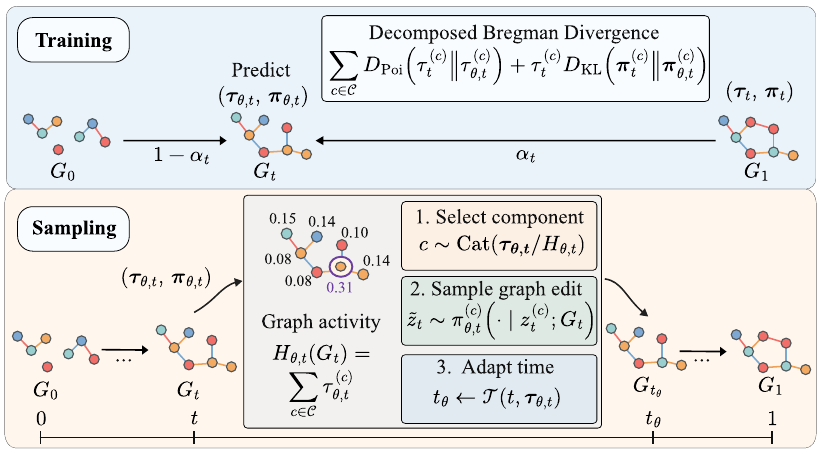}
    \caption{{Overview of \ours{}.}}
    \vspace{-1.5em}
    \label{fig:hglow-overview}
\end{figure}
\vspace{-0.4em}
\section{Background} 
\vspace{-0.5em}
\subsection{Continuous-Time Markov Chains (CTMC)}
A CTMC is a right-continuous stochastic process $\{z_t\}_{t \in [0,1]}$ taking values in a finite discrete state space $\mathcal{Z} = \{1, \dots, Z\}$ and satisfies the Markov property $p(z_t\mid z_{0:s})=p(z_t\mid z_s)$ for all $0\leq s<t\leq 1$. A CTMC is characterized by its \textit{holding times} and \textit{jumps}. Specifically, a state $z_t$ remains unchanged for a period of \textit{holding time}. When the holding time ends, the state jumps to the next state $z_{t'}$ according to a jump probability distribution. The holding-and-jumping process is then repeated. This two-step process is described using a generator matrix $\bm{R}_t \in \mathbb{R}^{Z \times Z}$, analogous to the velocity field in continuous state spaces \citep{lipman2022flow}, which defines transition rates between states:
\begin{equation}\label{eq:ctmc}
p_{t +\mathrm{d}t \mid t}(\tilde{z} \mid z_t)   =   \delta(z_t, \tilde{z}) + R_t(z_t, \tilde{z})\mathrm{d}t,
\end{equation}
where $\delta$ is the Kronecker delta, $R_t(z_t, \tilde{z})$ denotes an entry in the rate matrix, and $R_t(z_t,\tilde{z})\mathrm{d}t$ is the probability of state transition from $z_t$ to $\tilde{z}$ in  the
next infinitesimal time step $\mathrm{d}t$.
The generator matrix satisfies ${R}_t(z_t, \tilde{z}) \geq 0,$ for $z_t \neq \tilde{z}$, and ${R}_t(z_t, z_t) = - \sum_{\tilde{z} \neq z_t} {R}_t(z_t, \tilde{z})$ to ensure that the transition probabilities sum to 1. Any generator rate matrix can be factorized as:
\begin{equation}\label{eq:rate_decomposition}
    R_t(z_t,\tilde{z})
    =
    \hazard_t(z_t)\,\jump_t(\tilde{z}\mid z_t),
    \qquad \tilde{z}\neq z_t.
\end{equation}
Here, $\hazard_t(z_t)
    :=
    - {R}_t(z_t, z_t)
$ is the hazard of the state \(z_t\), i.e., the instantaneous rate of transitioning from $z_t$ to any other state, and $\jump_t(\tilde{z}\mid z_t)$ is the jump distribution, satisfying:
\begin{equation}
    \jump_t(\tilde{z}\mid z_t) \geq 0,
    \qquad
    \jump_t(z_t\mid z_t) = 0,
    \quad \text{and} \quad
    \sum_{\tilde{z}\neq z_t}
    \jump_t(\tilde{z}\mid z_t) = 1.
\end{equation}
A large hazard indicates a high instantaneous jump rate from \(z_t\), whereas \(\jump_t(\cdot\mid z_t)\) defines where it moves. In practice, the trajectory can be simulated with finite time intervals $\Delta t$ using Euler steps: 
\begin{equation}\label{eq:euler}
    z_{t+\Delta t}
    \sim
    \Cat\!\Big(
    \delta(z_t, \cdot)
    +
    R_{t}(z_t,\cdot) \Delta t
    \Big).
\end{equation} 

\subsection{Discrete flow matching}
A discrete flow model constructs a probability path \(p_t\), i.e., the generative flow, connecting a simple prior \(p_0 = p_\text{noise}\) to the target data distribution \(p_1 = p_\text{data}\). During inference, samples are generated by simulating a CTMC whose marginal distribution follows the probability path. To achieve this, we need to construct a rate matrix $\bm{R}_t$ that generates $p_t$, which is used to simulate the trajectory from $t=0$ to $t=1$ following \cref{eq:ctmc}. The marginal distribution and the generator matrix are
related by a conservation law, the \textit{Kolmogorov forward equation}:
\begin{equation}\label{eq:kolmogorov}
    \partial_t p_t(z_t)
    =
    \sum_{k\neq z_t} R_t(k,z_t)\,p_t(k)
    - \sum_{k\neq z_t} R_t(z_t,k)\,p_t(z_t),
\end{equation}
which expresses the time derivative of the marginal probability at $z_t$ as the net balance between incoming and
outgoing probability mass. However, since it is intractable to directly derive a generator for $p_t$, flow matching relies on a simpler datapoint conditional flow $p_{t\mid 1}(\cdot \mid z_1)$ satisfying:
\begin{equation}
p_t(z_t)
=
\mathbb{E}_{z_1 \sim p_{\mathrm{data}}}
\left[
p_{t\mid 1}(z_t \mid z_1)
\right],
\end{equation}
where the conditional flow starts from noise at $t=0$, and converges to the datapoint $z_1$ at $t=1$. Discrete flow matching (DFM) \citep{gat2024discrete} linearly interpolates towards $z_1$ as follows:
\begin{equation}\label{eq:noising_trajectory}
    p_{t\mid1}(z_t \mid z_1) = \alpha_t \delta(z_t, z_1) + (1 - \alpha_t) p_0(z_t), 
\end{equation}
where $\alpha_t$ is a scheduler increasing monotonically from $\alpha_0 = 0$ to $\alpha_1 = 1$. 

\textbf{Generator construction.} One choice of a valid conditional generator that generates the conditional probability path $p_{t\mid 1}(\cdot\mid z_1)$ is
\begin{equation}\label{eq:cond_rate}
R_t\!\left(z_t,\tilde z \mid z_1\right)
=
\frac{\dot{\alpha}_t}{1-\alpha_t}
\left[
\delta
(\tilde z,z_1)
-
\delta(z_t,\tilde z)
\right].
\end{equation}
\citet{gat2024discrete} show that a rate matrix generating the marginal path $p_t$ is obtained by averaging the conditional rates over the posterior $p_{1\mid t}(\cdot\mid z_t)$:
$
R_t(z_t,\tilde{z})
=
\mathbb{E}_{z_1 \sim p_{1\mid t}(\cdot\mid z_t)}
\left[
R_t(z_t,\tilde{z}\mid z_1)
\right].
$
\citet{gat2024discrete} derive the marginal rate matrix in closed form as follows:
\begin{equation}\label{eq:contruct_rate}
R_t(z_t,\tilde{z})
=
\frac{\dot{\alpha}_t}{1-\alpha_t}
\left[
p_{1\mid t}(\tilde{z} \mid z_t)-\delta(z_t, \tilde{z})
\right].
\end{equation}

\textbf{Learning.}
DFM trains a neural network $p^{\theta}_{1 \mid t} (\cdot \mid z_t)$ to approximate the true denoising distribution using the cross-entropy loss, i.e., learning to predict the clean data point $z_1$ given a noisy latent $z_t$:
\begin{equation}
\mathcal{L}_{\text{ce}}
= 
\mathbb{E}_{
t\sim\mathcal{U}(0,1),\,
z_1 \sim p_1(\cdot),\,
z_t \sim p_{t\mid 1}( \cdot\mid z_1)
}
\left[
-
\log
p_{1\mid t}^{\theta}(z_1 \mid z_t)
\right]
\end{equation}
Once trained, the generator is constructed as 
$
R_{\theta,t}(z_t,\tilde z)
=
\frac{\dot{\alpha}_t}{1-\alpha_t}
\left[
p^\theta_{1\mid t}(\tilde z\mid z_t)
-
\delta(z_t,\tilde z)
\right].
$
Together with the initial distribution $p_0$, this generator defines
a CTMC, which can be simulated using \cref{eq:euler}.
\section{Limitations of existing graph DFM models.}
Discrete flow matching has been recently extended for graphs \citep{qin2025defog}, where samples are generated by simulating a multivariate CTMC factorizing over node and edge attributes:
\begin{equation}
p_{t+\mathrm{d}t\mid t}
\bigl(\tilde{\bm{ z}}^{1:D}\mid \bm{z}_t^{1:D}\bigr)
=
\prod_{d=1}^{D}
\Bigl[
\delta\!\bigl(\tilde z^{(d)},z_t^{(d)}\bigr)
+
R_{\theta,t}^{(d)}\!
\bigl(z_t^{(d)},\tilde z^{(d)};G_t\bigr)\mathrm{d}t
\Bigr],
\label{eq:multivariate_ctmc}
\end{equation}
$t \in [0,1]$, $z^{(d)} \in \mathcal{Z}$ denotes a node or edge attribute, $\mathcal{Z}$ is the corresponding attribute space, $\bm{z}^{1:D}_t$ and $\tilde{\bm{z}}^{1:D} \in \mathcal{Z}^{D}$, and $D$ is the number of graph components.
In the following, we discuss the limitations of existing graph DFM models.

\textbf{Sampling complexity.}
Graph DFM models often rely on sophisticated sampling procedures, including predictor-corrector updates \citep{campbell2022continuous}, target guidance and additional stochasticity \citep{campbell2024generative}, as well as carefully designed sampling-time distortions \citep{qin2025defog}. In particular, \citet{luosimgfm,qin2025defog} design non-uniform time discretizations that concentrate computation in specific regions of the generation trajectory where graph validity is most sensitive. While effective, these techniques introduce additional hyperparameters and design choices, requiring substantial inference-time tuning for each task and sampling budget.

\textbf{Sampling efficiency.}
DFM unlocks flexible sampling strategies by decoupling model training from sampling. However, existing graph DFM methods still require a large number of function evaluations (NFE), often comparable to that of diffusion-based approaches.
In practice, DFM methods simulate the process in \cref{eq:multivariate_ctmc} using finite-step approximations such as $\tau$-leaping \citep{gillespie2001approximate}, which hold transition rates fixed over a finite interval and apply the resulting component updates simultaneously. These fixed rates neglect interactions between updates within the interval, thereby introducing approximation errors when multiple graph components change. These errors are more pronounced in the low-step regime, where larger intervals tend to produce more simultaneous updates. Accurately coordinating node and edge transitions is therefore crucial for maintaining sample quality at low NFE. 

\section{Hazard-Guided Sampling using \ours{}}\label{sec:sampling}

We introduce \emph{\ours{}}, a novel hazard-guided sampler for graph generation that builds on the decomposition of CTMC generators into hazards and conditional jump distributions. We first instantiate the decomposition for graphs in \cref{sec:graph_decomposed}, then present our \ours{} sampler in \cref{sec:graph_edits}.

\textbf{Notations.} We represent an undirected graph with $N$ nodes by its node and
edge attributes as $
    G
    =
    \left(
        \bm{x}^{1:n:N},
        \bm{e}^{1:i<j:N}
    \right),
$ where $\bm{x}^{1:n:N} = \{x^{(n)}\}_{n=1}^{N}$ and $\bm{e}^{1:i<j:N} = \{e^{(ij)}\}_{1\leq i < j \leq N}$. Here, $x^{(i)} \in \mathcal{X} = \{1,\ldots,X\}$ denotes
a categorical attribute of node $i$, and
$e^{(ij)} \in \mathcal{E} = \{1,\ldots,E\}$ denotes
a categorical edge attribute. Following \citet{vignac2022digress}, we associate an edge between every pair of nodes, and reserve one category in $\gE$ to represent the absence of an edge. $z^{(c)}$ denotes the attribute of a component $c$, corresponding to a node attribute $x^{(n)}$ or an edge attribute $e^{(ij)}$.


\subsection{Decomposition of graph generators}\label{sec:graph_decomposed}
In the following, we assume access to a model $f_\theta$ from which we can construct valid componentwise generators with transition rates  $R_{\theta,t}^{(c)}(z_t^{(c)},\tilde z;G_t)$. 
For each component $c$, we factorize the off-diagonal elements following \cref{eq:rate_decomposition} as: 
\begin{equation}\label{eq:component_generator}
R_{\theta,t}^{(c)}\!\left(z_t^{(c)},\tilde z;G_t\right)
=
\tau_{\theta,t}^{(c)}(z_t^{(c)}; G_t)
\,\pi_{\theta,t}^{(c)}\!\left(\tilde z\mid z_t^{(c)};G_t\right), \quad \tilde{z} \neq z_t^{(c)},
\end{equation}
where $\tau_{\theta,t}^{(c)}$ is the exit hazard and $\pi_{\theta,t}^{(c)}$ is the jump distribution. To ensure that each row of the generator sums to zero, we set the diagonal entries to:
\begin{equation}\label{eq:component_generator_diagonal}
    R_{\theta,t}^{(c)}\!\left(z_t^{(c)}, z_t^{(c)};G_t\right)
    =
    - \sum_{\tilde z\neq z_t^{(c)}}
    R_{\theta,t}^{(c)}
    \left(z_t^{(c)},\tilde z; G_t\right)
    = -
    \hazard_{\theta,t}^{(c)}\left(z_t^{(c)}; G_t\right)
\end{equation}
The resulting overall graph generator ${R}_{\theta,t}\left(G_t,\widetilde{G}\right)$ factorizes over nodes and edges, as shown in  \cref{prop:componentwise_graph_generator}.



\begin{proposition}[Graph generator]
\label{prop:componentwise_graph_generator}
The graph generator factorizes, i.e., it can be constructed as the sum of the generators of the individual component paths \citep{campbell2022continuous, xu2024discrete}:
\begin{equation}
\resizebox{0.98\linewidth}{!}{$\displaystyle
R_{\theta,t}(G_t,\widetilde G)
=
\sum_n
\delta\!\left(G_t^{\setminus(n)},\widetilde G^{\setminus(n)}\right)
R_{\theta,t}^{X,(n)}\!\left(x_t^{(n)},\tilde x;G_t\right)
+
\sum_{i<j}
\delta\!\left(G_t^{\setminus(ij)},\widetilde G^{\setminus(ij)}\right)
R_{\theta,t}^{E,(ij)}\!\left(e_t^{(ij)},\tilde e;G_t\right)
$}, \notag
\end{equation}
where $\tilde x^{(n)}$ and $\tilde e^{(ij)}$ denote the attributes of node $(n)$ and edge $(i,j)$ in $\widetilde G$, respectively, and $\delta\!\left(G_t^{\setminus(n)},
        \widetilde{G}^{\setminus(n)}
    \right)$ and $\delta\!\left(
        G_t^{\setminus(ij)},
        \widetilde{G}^{\setminus(ij)}
    \right)$ 
check if $G_t$ and $\widetilde{G}$ are identical except for node $(n)$ or edge $(ij)$, respectively. $\bm{R}^{X,(n)}_{\theta,t} \in \mathbb{R}^{{X} \times {X}}$ and $\bm{R}^{E,(ij)}_{\theta,t} \in \mathbb{R}^{{E} \times {E}}$ are the node and edge generators.
\end{proposition}

Under this construction, the graph CTMC changes one node or edge attribute at a time, with transition rates that depend on the entire graph $G_t$.



\begin{proposition}[Graph hazard]
\label{prop:total_graph_hazard}
For the generator in
\cref{prop:componentwise_graph_generator},
the total graph hazard is given by the sum of the componentwise node and edge hazards $\hazard_{\theta,t}^{X,(n)}$ and $\hazard_{\theta,t}^{E,(ij)}$ of $\bm{R}_{\theta,t}^{X,(n)}$ and $\bm{R}_{\theta,t}^{E,(ij)}$:
\begin{equation}
H_{\theta,t}\left(G_t\right)
=
\sum_{n}\hazard_{\theta,t}^{X,(n)}\left(x^{(n)}_t; G_t\right)
+
\sum_{i<j}\hazard_{\theta,t}^{E,(ij)}\left(e^{(ij)}_t; G_t\right)
\label{eq:predicted-graph-hazard}
\end{equation}
\end{proposition}
The graph hazard in \cref{eq:predicted-graph-hazard} provides a non-negative scalar measure of graph activity, with larger values indicating a higher instantaneous rate of graph edits. 

\vspace{-0,5em}
\setlength{\intextsep}{0pt}
\setlength{\columnsep}{6pt}
\begin{wrapfigure}{r}{0.57\textwidth}
\vspace{0\baselineskip} 
\begin{minipage}{\linewidth}
\begin{algorithm}[H]
\caption{\ours{}: Hazard-guided sampling}
\label{alg:hazard-single-component}
\footnotesize
\begin{algorithmic}[1]
\State \textbf{Input:} model $f_\theta$, budget $T$,
threshold $T_{\mathrm{stop}}$
\State Initialize $G\sim p_0$, $t\gets0$, $k\gets0$
\State Let $\mathcal C=\mathcal C_X\cup\mathcal C_E$
\While{$k<T$}
    \State Construct $\bm{R}_{\theta,t}^{(c)}$ and derive
    $\{(\hazard_{\theta,t}^{(c)},\boldsymbol{\jump}_{\theta,t}^{(c)})\}_{c\in\mathcal C}$ \Comment{\cref{eq:component_generator_diagonal}}

    \State $H_{\theta,t}\gets\displaystyle\sum_{c\in\mathcal C}\hazard_{\theta,t}^{(c)}$
    \Comment{\cref{eq:predicted-graph-hazard}}

    \State \textbf{if} $H_{\theta,t}<T_{\mathrm{stop}}$
        \textbf{then} \Return $G$
    \Comment{Early stopping}


    \State
    $
    c\sim\operatorname{Cat}
    \left(
        \left\{{\hazard_{\theta,t}^{(c)}}/{H_{\theta,t}}\right\}_{c\in\mathcal C}
    \right)
    $
    \Comment{Component selection}

    \State 
    $\widetilde z\sim {\jump}_{\theta,t}^{(c)}(\cdot \mid z; G)$
    \Comment{State jump \cref{eq:hazard-conditional-jump}}

    \State
    $G\gets\operatorname{Update}(G,c,\widetilde z)$
    \Comment{Update the graph}

    \State 
    $
    t\gets
        1-\frac{1-\alpha_t}{2\dot{\alpha}_t}\left(
            \overline{\hazard}_{\theta,t}^{X}
            +
            \overline{\hazard}_{\theta,t}^{E}
            \right)
    $
    \Comment{Hazard-guided clock}


    \State $k\gets k+1$
\EndWhile

\State \Return $G$
\end{algorithmic}
\end{algorithm}
\end{minipage}
\vspace{-2\baselineskip} 
\end{wrapfigure}

\subsection{Graph generation planning under \ours{}}\label{sec:graph_edits}
Building on the hazard-jump decomposition, we design the \ours{} sampler to fully leverage the model’s own predictions for graph-edit planning. The predicted hazards and conditional jump distributions jointly guide three key decisions: \textit{(i)} which nodes or edges to update, \textit{(ii)} which states the selected components should transition to, and \textit{(iii)} at which time step to evaluate the model next.

\textbf{Component selection.}
\textit{Hazard-based selection.}
A natural choice is to sample components based on their contribution to the graph hazard, where, for $H_{\theta,t}>0$,:
\begin{equation}
c_t\sim\operatorname{Cat}\!\left(
\left\{
{\hazard_{\theta,t}^{(c)}}/{H_{\theta,t}}
\right\}_{c\in\mathcal C}
\right), 
\label{eq:hazard-component-selection}
\end{equation} 

\textit{Confidence-weighted selection.}
Alternatively, we can select components with confident jumps:
\begin{equation}
w_{\theta,t}^{(c)}
=
\hazard_{\theta,t}^{(c)}
\max_{\widetilde z\neq z_t^{(c)}}
\pi_{\theta,t}^{(c)}
\!\left(\widetilde z\mid z_t^{(c)};G_t\right),
\label{eq:confidence-weighted-hazard}
\end{equation}
and sample in proportion to $w_{\theta,t}^{(c)}$. We provide an ablation of component selection criteria in \cref{tab:sampling_ablation}.

\textbf{Conditional jump.} Once $c_t$ is selected for an update, we sample its jump following:
\begin{equation}
\widetilde z \sim
\jump_{\theta,t}^{(c_t)}
\!\left(\cdot\mid z_t^{(c_t)};G_t\right).
\label{eq:hazard-conditional-jump}
\end{equation}
As this jump excludes the current state, each conditional jump changes exactly one component. 


\textbf{Hazard-based time allocation.} Sampling time distortion is particularly important when generating graphs under a limited number of steps $T$. 
\ours{} benefits from the predicted hazards to \textit{dynamically} adapt the time according to the graph activity. Higher graph activity motivates a finer discretization with more refinement steps to revise earlier predictions, whereas low activity motivates a coarser discretization with fewer refinement steps. 

To derive a sampling-time policy, we first look into the hazards implied by the DFM generator. Under the DFM  construction in \cref{eq:contruct_rate}, the conditional hazard simplifies to:  
$
\hazard_t^{(c)}(G_t \mid G_1)
=
\frac{\dot{\alpha}_t}{1-\alpha_t}\,\mathbf{1}\!\left\{
z_t^{(c)} \neq z_1^{(c)}
\right\}.
$
The marginal hazard averages conditional hazards over clean graphs:
\begin{equation}
\hazard_t^{(c)}(G_t) =
\mathbb{E}_{G_1\mid G_t}
\left[\hazard_t^{(c)}(G_t\mid G_1)\right] =
\frac{\dot{\alpha}_t}{1-\alpha_t}\,\Pr\!\left(
z_1^{(c)}\neq z_t^{(c)}\mid G_t
\right).
\label{eq:hazard-mismatch-interpretation}
\end{equation}
Therefore, hazards estimate the remaining componentwise disagreement between the noisy latent $G_t$ and the clean graph. We leverage this to construct the following hazard-guided sampling time:
\begin{equation}
\widehat t
=
1-\frac{1-\alpha_t}{2\dot\alpha_t}
\left(
\frac{1}{N}\sum_{i=1}^{N}\hazard_{\theta,t}^{X,(i)}
+
\frac{1}{\binom{N}{2}}
\sum_{1\leq i<j\leq N}\hazard_{\theta,t}^{E,(ij)}
\right)
\label{eq:hazard-target-time}
\end{equation}
This measures the average remaining graph activity. We note that we normalize node and edge hazards to prevent
edge activity from dominating the clock update and to avoid larger time jumps in larger graphs. We provide the full derivation in \cref{app:hazard-time}. Importantly, the time allocation in \cref{eq:hazard-target-time} defines a graph-dependent sequence of evaluations that can move forward or backward in time. 
When time moves from $t$ to $s < t$, we renoise the graph by independently resampling nodes and edges from their corresponding priors with probability $1-\frac{\alpha_s}{\alpha_t}$, allowing the sampler to revisit earlier predictions.

\textbf{Early stopping.} Once the graph hazard falls below a threshold $T_{\mathrm{stop}}$, we terminate the trajectory. This prevents further graph edits when the model predicts low graph activity, reducing the risk of disrupting an
otherwise valid graph. This is particularly crucial for graph generation, where even a single incorrect node or edge update can violate structural or chemical constraints.
Our ablation in \cref{tab:sampling_ablation} shows that removing early stopping degrades sample quality. In \cref{sec:experiments}, we show that \ours{} terminates before exhausting the entire budget $T$ while maintaining strong performance.

\vspace{-0,5em}
\subsection{Accelerating pretrained models}
\ours{} can be applied directly to pretrained models with accessible CTMC generators. Existing graph DFM models train a denoising network $p_{1\mid t}^{\theta}(\cdot\mid G_t)$ to predict distributions over clean node and edge attributes, and construct generators following \cref{eq:contruct_rate}. 
These models typically sample via Euler discretization, updating node and edge states independently in parallel at each time step. 
In contrast, \ours{} decomposes the resulting generators post hoc into hazards and conditional jump distributions to guide graph edits. In \cref{fig:defog_hglow_comparison}, we show that \ours{} can significantly accelerate graph generation without additional training or sampling-time interventions. We summarize our sampling strategy in \cref{alg:hazard-single-component}, omitting the generator construction and renoising step for clarity.
\section{\ours{}: Discrete Flow matching Model for Graphs}\label{sec:training}
\begin{wrapfigure}{r}{0.42\textwidth}
\vspace{0\baselineskip} 
\begin{minipage}{\linewidth}
\begin{algorithm}[H]
\caption{Training the \ours{} model}
\label{alg:cat-diffusion-training}
\footnotesize
\begin{algorithmic}[1]
    \State \textbf{Input:} Graph dataset
    $\mathcal{D}$

    \While{$f_\theta$ has not converged}
        \State Sample $G_1\sim\mathcal{D}$ 
        \State Sample $t\sim\mathcal{T}$
        \State Sample $G_0\sim p_0$
        \State Sample
        $G_t\sim p_{t\mid 1}(\cdot\mid G_1)$
        \State Compute $(\boldsymbol{\hazard}_{t}^{X},
    \boldsymbol{\jump}_{t}^{X},
    \boldsymbol{\hazard}_{t}^{E},
    \boldsymbol{\jump}_{t}^{E})$ 

        \State
        $(\boldsymbol{\hazard}_{\theta,t}^{X},
    \boldsymbol{\jump}_{\theta,t}^{X},
    \boldsymbol{\hazard}_{\theta,t}^{E},
    \boldsymbol{\jump}_{\theta,t}^{E}) \gets f_\theta(G_t,t)$ 

        \State loss $\gets \mathcal{L}(\theta)$
        \State optimizer.step(loss)
    \EndWhile
\end{algorithmic}
\end{algorithm}
\end{minipage}
\vspace{-2\baselineskip} 
\end{wrapfigure}
Motivated by the acceleration gains achieved by \ours{} using pretrained models, we propose a structured generator-learning framework for graphs. In the following, we propose the \textit{\ours{}} model, a simple discrete flow-matching model for graph generation that directly parameterizes the generator using hazards and conditional jump distributions.

\textbf{Noising.}
We design the noising trajectory by independently adding noise to each node and edge:
\begin{align}\label{eq:noise_graph}
    &p_{t\mid1}\left(G_t \mid G_1\right) = \notag \\ &\quad \prod_{n} p_{t\mid1} \left(x^{(n)}_t \mid x^{(n)}_1\right)\prod_{i<j} p_{t\mid1}\left(e^{(ij)}_t \mid e^{(ij)}_1\right),
\end{align}
where $p_{t\mid1} \left(x^{(n)}_t \mid x^{(n)}_1\right)$ and $p_{t\mid1}\left(e^{(ij)}_t \mid e^{(ij)}_1\right)$ are defined following \cref{eq:noising_trajectory}. 
We use empirical marginal priors over node and edge types $p^{\mathcal{X}}_0$ and $p^{\mathcal{E}}_0$ \citep{vignac2022digress}. 

\textbf{Conditional targets.} 
For each component $c$, we factorize the conditional generator $R_t(z_t^{(c)},\tilde z\mid G_1)$ into an exit hazard $\hazard_t^{(c)}(z_t^{(c)}\mid G_1)$ and conditional jump distribution $\jump_t^{(c)}(\cdot\mid z_t^{(c)},G_1)$.

\textbf{Optimization.}
We parameterize node and edge hazards and conditional jump distributions using a neural network:
$
f_\theta(G_t, t)
=
\left(
    \boldsymbol{\hazard}_{\theta,t}^{X},
    \boldsymbol{\jump}_{\theta,t}^{X},
    \boldsymbol{\hazard}_{\theta,t}^{E},
    \boldsymbol{\jump}_{\theta,t}^{E}
\right).
$
We follow the conditional generator matching framework in \citet{holderrieth2025generator}, and use the Bregman Divergence loss to optimize the continuous-time ELBO derived for finite-state jump processes:
\begin{align}
\mathcal{L}(\theta) = \mathbb{E}_{\substack{G_1\sim p_{1},\;\\t\sim\mathcal{T},\\ G_t\sim p_{t\mid 1}\left(\cdot\mid G_1\right)}} \left[
\sum_{\widetilde{G}\neq G_t} 
\left( 
R_{\theta,t}\left(G_t, \widetilde{G}\right) - R_t\left(G_t, \widetilde{G}\mid G_1\right)
\log R_{\theta,t}\left(G_t, \widetilde{G}\right) 
\right)
\right] + C
\end{align}
We show in \cref{thm:decomposed-gm} that, under \ours{}'s factorization, this objective separates exactly into a hazard matching and a hazard-weighted jump optimization over nodes and edges.

\begin{theorem}[Decomposed graph generator matching objective]
\label{thm:decomposed-gm}
Assume that the conditional target and learned graph
generators follow the single-component construction in
\cref{prop:componentwise_graph_generator} and admit hazard-jump decompositions. Then, the conditional generator matching objective decomposes as:
\begin{equation}
\mathcal{L}(\theta)
=
\mathbb{E}_{\substack{
G_1\sim p_1,\;t\sim\mathcal{T},\\
G_t\sim p_{t\mid 1}(\cdot\mid G_1)
}}
\left[
\sum_{n=1}^{N}\mathcal{L}_{\theta,t}^{X,(n)}
+
\sum_{1\le i<j\le N}\mathcal{L}_{\theta,t}^{E,(ij)}
\right],
\label{eq:componentwise-decomposed-loss-full}
\end{equation}
where, for each graph component $c$, 
\begin{equation}
\mathcal{L}_{\theta,t}^{(c)}
=
D_{\mathrm{Poi}}\!\left(
\hazard_t^{(c)}
\,\middle\|\,
\hazard_{\theta,t}^{(c)}
\right)
+
\hazard_t^{(c)}
D_{\mathrm{KL}}\!\left(
\boldsymbol{\jump}_t^{(c)}
\,\middle\|\,
\boldsymbol{\jump}_{\theta,t}^{(c)}
\right).
\label{eq:componentwise-decomposed-loss}
\end{equation}
Here, $D_{\mathrm{Poi}}(a\|b)=a\log\frac{a}{b}-a+b$.
\end{theorem}

We provide a proof in \cref{app:objective}. 
Theorem~\ref{thm:decomposed-gm} shows that learning the graph generator reduces to two separate prediction tasks for every node and edge. The hazard head learns the instantaneous rate at which the component changes, whereas the jump head learns the destination distribution conditional on a jump occurring. Importantly, the jump loss is weighted by the target hazard. Therefore, components with a high hazard receive greater weight, whereas the jump loss vanishes for components with zero target hazard. We summarize the training algorithm in \cref{alg:cat-diffusion-training}.

\section{Experiments}\label{sec:experiments} 
\vspace{-0,5em}
We evaluate \ours{} on molecular and synthetic graph benchmarks. Unless otherwise specified, \ours{} results refer to the \ours{} DFM model trained with the decomposed hazard–jump objective (\cref{sec:training}) and sampled using the \ours{} sampler (\cref{sec:sampling}).
To isolate the effect of the network architecture, we build the \ours{} model using the graph transformer in \citet{vignac2022digress}.
We describe the benchmarks, baselines, and model architectures in \cref{app:exp_details}.


\setlength{\intextsep}{3pt}
\setlength{\columnsep}{6pt}

\begin{wraptable}{r}{0.46\textwidth}
\centering
\setlength{\tabcolsep}{1pt}
\vspace{-1\baselineskip}
\caption{Performance on QM9. We report \emph{Steps}, denoting the available sampling budget $T$, and the mean $\pm$ standard deviation of NFE across sampled molecules.}
\label{tab:qm9_without_hydrogens}
\vspace{-0.8em}

\resizebox{\linewidth}{!}{%
\begin{tabular}{lc
    S[table-format=2.1]
    S[table-format=2.1]
    S[table-format=3.1]
    S[table-format=1.2]
}
\toprule
Model
& Steps (NFE)
& {Valid $\uparrow$}
& {\shortstack{Relaxed \\Validity $\uparrow$}}
& {Unique $\uparrow$}
& {FCD $\downarrow$} \\
\midrule
Training set
& {--} & 99.3 & 99.5 & 99.2 & 0.03 \\
\midrule

DiGress \tiny~\citep{vignac2022digress}
& 500\,{\scriptsize(500 $\pm$ 0.0)}
& 99.0 & {--} & 96.2 & {--} \\

CatFlow \tiny~\citep{eijkelboom2024variational}
& {--}
& 99.8 & {--} & 100.0 & 0.44 \\

DisCo \tiny~\citep{xu2024discrete}
& 500\,{\scriptsize(500 $\pm$ 0.0)}
& 99.3 & {--} & {--} & {--} \\

Cometh\tiny~\citep{siraudin2024cometh}
& 500\,{\scriptsize(500 $\pm$ 0.0)}
& 99.6 & {--} & 96.8 & 0.25 \\

SID \tiny~\citep{boget2025simple}
& 500\,{\scriptsize(500 $\pm$ 0.0)}
& 99.7 & {--} & {--} & 0.50 \\

CID\tiny~\citep{boget2025simple}
& 500\,{\scriptsize(500 $\pm$ 0.0)}
& 99.9 & {--} & {--} & 1.76 \\
\midrule

SimGFM\tiny~\citep{luosimgfm}
& 50\,{\scriptsize(50 $\pm$ 0.0)}
& 99.7 & 99.8 & 96.3 & 0.13 \\

SimGFM\tiny~\citep{luosimgfm}
& 500\,{\scriptsize(500 $\pm$ 0.0)}
& 99.8 & 99.9 & 94.9 & 0.31 \\
\midrule

DeFoG\tiny~\citep{qin2025defog}
& 50\,{\scriptsize(50 $\pm$ 0.0)}
& 98.9 & 99.2 & 96.2 & 0.26 \\

DeFoG\tiny~\citep{qin2025defog}
& 500\,{\scriptsize(500 $\pm$ 0.0)}
& 99.3 & 99.4 & 96.3 & 0.12 \\
\midrule

\rowcolor{gray!15}
\textbf{\ours{} (Ours)}
& 20\,{\scriptsize(19.0 $\pm$ 0.4)}
& 99.2 & 99.5 & 95.0 & 0.37 \\

\rowcolor{gray!15}
\textbf{\ours{} (Ours)}
& 50\,{\scriptsize(22.4 $\pm$ 2.3)}
& 99.1 & 99.5 & 96.2 & 0.10 \\
\bottomrule
\end{tabular}%
}

\vspace{-1em}
\end{wraptable}
\subsection{Molecular graph generation}\label{sec:molecules_exps}

\vspace{-0,5em}
\textbf{Small molecule generation.}
We evaluate \ours{} on QM9 \citep{wu2018moleculenet} and report the results in \cref{tab:qm9_without_hydrogens}. While \emph{Steps} denotes the available budget of $T$ sampling steps, \ours{} can terminate well before exhausting this budget when graph activity becomes sufficiently low, resulting in substantially fewer function evaluations. We report the mean $\pm$ standard deviation of NFE across sampled molecules in parentheses. With a budget of only 50 sampling steps, \ours{} achieves the lowest FCD among the evaluated models, surpassing both 50- and 500-step baselines while maintaining high validity.


\begin{table}[t]
\centering
\setlength{\tabcolsep}{1.5pt}
\caption{Large molecule generation performance on MOSES and GuacaMol. Only iterative denoising-based methods are reported. We report \emph{Steps}, denoting the sampling budget $T$, and the mean $\pm$ standard deviation of NFE across sampled molecules. }
\vspace{-0.5em}
\label{tab:large_molecule_gen}

\resizebox{\textwidth}{!}{%
\begin{tabular}{@{}lc
    *{4}{S[table-format=3.1]}
    *{2}{S[table-format=1.2]}
    S[table-format=2.1]
    |c*{5}{S[table-format=3.1]}
@{}}
\toprule
& \multicolumn{8}{c|}{MOSES}
& \multicolumn{6}{c}{GuacaMol} \\
\cmidrule(lr){2-9}\cmidrule(lr){10-15}

Model
& Steps (NFE)
& {Val.\,$\uparrow$}
& {Unique.\,$\uparrow$}
& {Novelty\,$\uparrow$}
& {Filters\,$\uparrow$}
& {FCD\,$\downarrow$}
& {SNN\,$\uparrow$}
& {Scaf.\,$\uparrow$}
& Steps (NFE)
& {Val.\,$\uparrow$}
& {V.U.\,$\uparrow$}
& {V.U.N.\,$\uparrow$}
& {KLD.\,$\uparrow$}
& {FCD\,$\uparrow$} \\
\midrule

Training set
& --
& 100.0 & 100.0 & 0.0 & 100.0 & 0.01 & 0.64 & 99.1
& --
& 100.0 & 100.0 & 0.0 & 99.9 & 92.8 \\
\midrule

DiGress{\tiny~\citep{vignac2022digress}}
& 500\,{\scriptsize(500 $\pm$ 0.0)}
& 85.7 & 100.0 & 95.0 & 97.1 & 1.19 & 0.52 & 14.8
& 500\,{\scriptsize(500 $\pm$ 0.0)}
& 85.2 & 85.2 & 85.1 & 92.9 & 68.0 \\

DisCo{\tiny~\citep{xu2024discrete}}
& 500\,{\scriptsize(500 $\pm$ 0.0)}
& 88.3 & 100.0 & 97.7 & 95.6 & 1.44 & 0.50 & 15.1
& 500\,{\scriptsize(500 $\pm$ 0.0)}
& 86.6 & 86.6 & 86.5 & 92.6 & 59.7 \\

Cometh{\tiny~\citep{siraudin2024cometh}}
& 500\,{\scriptsize(500 $\pm$ 0.0)}
& 86.1 & 100.0 & 96.0 & 97.3 & 1.27 & 0.50 & 16.4
& 500\,{\scriptsize(500 $\pm$ 0.0)}
& 94.4 & 94.4 & 93.5 & 94.1 & 67.4 \\

Cometh-PC{\tiny~\citep{siraudin2024cometh}}
& 500\,{\scriptsize(500 $\pm$ 0.0)}
& 90.5 & 99.9 & 92.6 & 99.1 & 1.27 & 0.54 & 16.0
& 500\,{\scriptsize(500 $\pm$ 0.0)}
& 98.9 & 98.9 & 97.6 & 96.7 & 72.7 \\
\midrule

SimGFM{\tiny~\citep{luosimgfm}}
& 50\,{\scriptsize(50 $\pm$ 0.0)}
& 86.1 & 100.0 & 95.1 & 98.6 & 1.44 & 0.51 & 16.9
& 50\,{\scriptsize(50 $\pm$ 0.0)}
& 98.5 & 98.5 & 98.3 & 93.6 & 77.1 \\

SimGFM{\tiny~\citep{luosimgfm}}
& 200\,{\scriptsize(200 $\pm$ 0.0)}
& 89.4 & 100.0 & 94.6 & 98.9 & 1.08 & 0.55 & 17.3
& 200\,{\scriptsize(200 $\pm$ 0.0)}
& 99.2 & 99.2 & 98.5 & 94.6 & 76.4 \\

SimGFM{\tiny~\citep{luosimgfm}}
& 500\,{\scriptsize(500 $\pm$ 0.0)}
& 91.1 & 100.0 & 93.9 & 99.2 & 1.19 & 0.56 & 14.4
& 500\,{\scriptsize(500 $\pm$ 0.0)}
& 99.2 & 99.2 & 98.6 & 95.5 & 73.5 \\
\midrule

DeFoG{\tiny~\citep{qin2025defog}}
& 50\,{\scriptsize(50 $\pm$ 0.0)}
& 82.9 & 99.9 & 95.1 & 97.8 & 1.79 & 0.53 & 19.8
& 50\,{\scriptsize(50 $\pm$ 0.0)}
& 91.7 & 91.7 & 91.2 & 92.3 & 57.9 \\

DeFoG{\tiny~\citep{qin2025defog}}
& 500\,{\scriptsize(500 $\pm$ 0.0)}
& 91.5 & 99.9 & 90.0 & 99.2 & 1.47 & 0.55 & 11.5
& 500\,{\scriptsize(500 $\pm$ 0.0)}
& 99.0 & 99.0 & 97.9 & 97.7 & 73.8 \\
\midrule

\rowcolor{gray!15}
\textbf{\ours{} (Ours)}
& 50\,{\scriptsize(35.9 $\pm$ 7.0)}
& 89.5 & 100.0 & 94.4 & 99.1 & 0.98 & 0.55 & 18.8
& 50\,{\scriptsize(40.4 $\pm$ 8.1)}
& 95.2 & 95.2 & 93.6 & 89.9 & 69.3 \\

\rowcolor{gray!15}
\textbf{\ours{} (Ours)}
& 100\,{\scriptsize(47.1 $\pm$ 15.7)}
& 89.9 & 99.9 & 93.0 & 99.2 & 0.96 & 0.56 & 16.0
& 100\,{\scriptsize(49.2 $\pm$ 13.7)}
& 99.2 & 99.2 & 98.1 & 96.9 & 77.7 \\

\rowcolor{gray!15}
\textbf{\ours{} (Ours)}
& 200\,{\scriptsize(57.2 $\pm$ 14.2)}
& 91.9 & 99.9 & 93.1 & 99.4 & 1.08 & 0.58 & 15.5
& 200\,{\scriptsize(78.8 $\pm$ 35.6)}
& 99.4 & 99.4 & 97.9 & 96.8 & 77.3 \\

\bottomrule
\end{tabular}%
}
\vspace{-1.0em}
\end{table}
\textbf{Large molecule generation.}
We evaluate \ours{} on GuacaMol \citep{brown2019guacamol} and Moses \citep{polykovskiy2020molecular} and report results in \cref{tab:large_molecule_gen}. We compare with recent discrete diffusion and flow matching models. \ours{} achieves the best FCD on both datasets with only 100 steps, outperforming the 500-step baselines. 

\vspace{-0.5em}
\newcommand{\stdsmall}[1]{{\scriptstyle\,\pm\,#1}}
\setlength{\intextsep}{3pt}

\setlength{\columnsep}{6pt}
\begin{wraptable}{r}{0.48\textwidth}
\centering
\setlength{\tabcolsep}{1pt} 
\vspace{0\baselineskip}
\caption{Performance on synthetic graphs. Full version in \cref{tab:app_synthetic_graph_gen}.
}
\vspace{-0,5em}
\label{tab:synthetic_graph_gen}

\sisetup{
    table-format = 3.1,
    table-space-text-post = {$\stdsmall{0.0}$},
    table-align-text-post = false
}

\resizebox{\linewidth}{!}{%
\begin{tabular}{@{}lc*{2}{S}|*{2}{S}@{}}
\toprule
& & \multicolumn{2}{c}{Planar}
& \multicolumn{2}{c}{Tree}
\\
\cmidrule(lr){3-4}\cmidrule(lr){5-6}

Model
& Steps
& {V.U.N.\,$\uparrow$} & {Ratio\,$\downarrow$}
& {V.U.N.\,$\uparrow$} & {Ratio\,$\downarrow$}
\\
\midrule

Training set
& {--}
& 100.0 & 1.0
& 100.0 & 1.0
\\

\midrule






EDGE\tiny~\citep{chen2023efficient}
& 1000
& 0.0 & 431.4
& 0.0 & 850.7
\\


DiGress{\tiny~\citep{vignac2022digress}}
& 1000
& 77.5 & 5.1
& 90.0 & 1.6
\\

HSpectre\tiny~\citep{bergmeister2024efficient}
& {--}
& 95.0 & 2.1
& 100.0 & 4.0
\\

GruM\tiny~\citep{jo2023graph}
& {--}
& 90.0 & 1.8
& {--} & {--}
\\

DisCo{\tiny~\citep{xu2024discrete}}
& 500
& 83.6 & {--}
& {--} & {--}
\\

Cometh{\tiny~\citep{siraudin2024cometh}}
& 500
& 92.5 & {--}
& {--} & {--}
\\

Cometh-PC{\tiny~\citep{siraudin2024cometh}}
& {--}
& 99.5 & {--}
& {--} & {--}
\\

CatFlow\tiny~\citep{eijkelboom2024variational}
& {--}
& 80.0 & {--}
& {--} & {--}
\\

\midrule

SimGFM{\tiny~\citep{luosimgfm}}
& 20
& 94.0 {$\stdsmall{4.4}$} & 2.3 {$\stdsmall{0.6}$}
& 88.0 {$\stdsmall{4.8}$} & 2.5 {$\stdsmall{0.9}$}
\\

SimGFM{\tiny~\citep{luosimgfm}}
& 50
& 99.5 {$\stdsmall{1.0}$} & 1.8 {$\stdsmall{0.5}$}
& 97.0 {$\stdsmall{1.0}$} & 2.0 {$\stdsmall{0.7}$}
\\

SimGFM{\tiny~\citep{luosimgfm}}
& 200
& 100.0 {$\stdsmall{0.0}$} & 9.3 {$\stdsmall{2.6}$}
& 99.5 {$\stdsmall{1.0}$}
& 1.5 {$\stdsmall{0.2}$}
\\

SimGFM{\tiny~\citep{luosimgfm}}
& 1000
& 100.0 {$\stdsmall{0.0}$} & 21.0 {$\stdsmall{1.5}$}
& 100.0 {$\stdsmall{0.0}$} & 2.1 {$\stdsmall{0.8}$}
\\

\midrule

DeFoG{\tiny~\citep{qin2025defog}}
& 50
& 95.0 {$\stdsmall{3.2}$} & 3.2 {$\stdsmall{1.1}$}
& 73.5 {$\stdsmall{9.0}$} & 2.5 {$\stdsmall{1.0}$}
\\

DeFoG{\tiny~\citep{qin2025defog}}
& 1000
& 99.5 {$\stdsmall{1.0}$}
& 1.6 {$\stdsmall{0.4}$}
& 96.5 {$\stdsmall{2.6}$}
& 1.6 {$\stdsmall{0.4}$}
\\

\midrule

\rowcolor{gray!15}
\textbf{\ours{} (Ours)}
& 20
& 95.0 {$\stdsmall{3.5}$} & 1.9 {$\stdsmall{0.5}$}
& 93.5 {$\stdsmall{4.6}$} & 2.2 {$\stdsmall{0.8}$}
\\

\rowcolor{gray!15}
\textbf{\ours{} (Ours)}
& 50
& 99.0 {$\stdsmall{1.2}$} & 1.9 {$\stdsmall{0.4}$}
& 98.0 {$\stdsmall{2.9}$} & 2.7 {$\stdsmall{0.9}$}
\\

\rowcolor{gray!15}
\textbf{\ours{} (Ours)}
& 100
& 99.0 {$\stdsmall{1.2}$} & 1.8 {$\stdsmall{0.3}$}
& 98.5 {$\stdsmall{2.0}$} & 1.8 {$\stdsmall{0.4}$}
\\

\rowcolor{gray!15}
\textbf{\ours{} (Ours)}
& 200
& 100.0 {$\stdsmall{0.0}$} & 2.4 {$\stdsmall{0.2}$}
& 100.0 {$\stdsmall{0.0}$} & 2.0 {$\stdsmall{0.8}$}
\\

\bottomrule
\end{tabular}%
\par}
\vspace{-1,2em}
\end{wraptable}
\subsection{Synthetic graphs}
The sampler in \cref{alg:hazard-single-component} achieves strong empirical performance on both small and large molecular graphs using few NFEs, see \cref{sec:molecules_exps}. 
This is possible because the sampling initialization from the empirical node and edge priors $p_0^{X}$ and $p_0^{E}$ enforces sparsity, where most candidate edges are assigned to the no-bond category, and only a few graph edits can be sufficient to generate a valid molecule. 
For graphs requiring more extensive structural changes, where the number of required graph edits exceeds the available NFE budget, we extend \ours{} to perform up to $B$ edits in parallel at each step; see \cref{alg:hazard-parallel}. 
We evaluate on Planar \citep{martinkus2022spectre} and Tree \citep{bergmeister2024efficient}, reporting means and standard deviations over 5 runs of 40 graphs (\cref{tab:synthetic_graph_gen}). With 20 steps, \ours{} achieves 95.0\% and 93.5\% V.U.N. on Planar and Tree, respectively, and 100\% at 200 steps with competitive Ratios.

\vspace{-0.5em}
\subsection{Accelerating pretrained DFM models}
\ours{} exploits the generator decomposition to accelerate pretrained graph DFMs as a drop-in sampler, without retraining or added sampling complexity. 
We evaluate our hazard-guided \ours{} sampler on the pretrained DeFoG model \citep{qin2025defog}.
We compare against two sampling configurations from DeFoG: base parallel Euler sampling and Euler sampling with additional target guidance and stochasticity, controlled by $(\eta,\omega)$, see \cref{fig:defog_hglow_comparison}. 
\ours{} matches or exceeds the validity and FCD performance of 500-step Euler sampling using only 100 steps on GuacaMol and MOSES. With 200 steps, \ours{} achieves quality comparable to the 500-step guided Euler configuration. 
Importantly, our framework significantly improves performance using DeFoG’s base generator, without requiring sampling-time generator modifications that can expand the hyperparameter space.
The full quantitative comparison is provided in \cref{tab:app_large_molecule_gen}.
\begin{figure}[t]
    \centering
    \includegraphics[width=\textwidth]{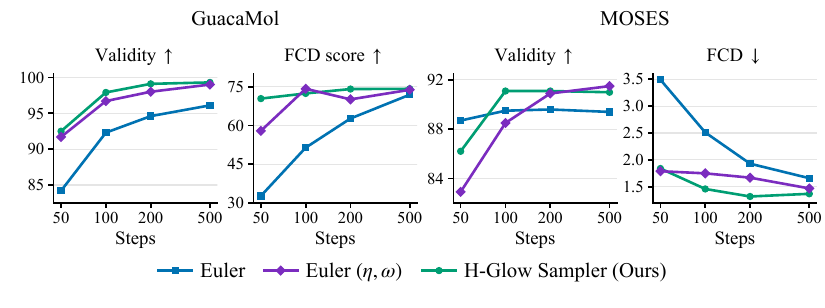}
    \vspace{-2,5em}
    \caption{
    \textbf{Sampler comparison using the same pretrained DeFoG model on GuacaMol and MOSES.} We compare base parallel Euler sampling (blue), Euler sampling with DeFoG’s added target guidance and stochasticity $(\eta,\omega)$ (purple), and our \ours{} sampler applied to DeFoG’s base generator (green). Our \ours{} sampling strategy achieves high sample quality with substantially fewer sampling steps.
    }
    \vspace{-1,2em}
    \label{fig:defog_hglow_comparison}
\end{figure}

\subsection{Ablations}
\textbf{Sample time distortion.}
Uniform time discretization is suboptimal for graph generation, as specific regions of the generation trajectory can benefit from finer control. 
Prior work employs specialized 
\setlength{\intextsep}{3pt}

\setlength{\columnsep}{6pt}
\begin{wrapfigure}{r}{0.35\textwidth}
    \centering
    \vspace{0\baselineskip}
    \begingroup
    \captionof{table}{Clock ablations.}
    \vspace{-0,8em}
    \label{tab:planar_clock_ablation}
    \scriptsize
    \setlength{\tabcolsep}{1pt}
    \renewcommand{\arraystretch}{1.1}

    \resizebox{\linewidth}{!}{%
    \begin{tabular}{lcc|cc}
    \toprule
    & \multicolumn{2}{c}{$T=50$}
    & \multicolumn{2}{c}{$T=100$} \\
    \cmidrule(lr){2-3}\cmidrule(lr){4-5}
    Clock
    & V.U.N. $\uparrow$ & Ratio $\downarrow$
    & V.U.N. $\uparrow$ & Ratio $\downarrow$ \\
    \midrule
    \textbf{Hazard-guided}
    & 99.0 & {1.86}
    & 99.0 & {1.81} \\
    \midrule
    Uniform
    & 99.5 & 23.43
    & {100.0} & 49.11 \\
    Polydec-2
    & {100.0} & 20.20
    & {100.0} & 42.62 \\
    Polydec-10
    & 99.5 & 5.57
    & 99.5 & 12.66 \\
    Polydec-20
    & 95.5 & 3.88
    & 99.5 & 6.08 \\
    \bottomrule
    \end{tabular}%
    }
    \endgroup
    \vspace{-1em}
\end{wrapfigure}
time distortions to redistribute steps along the generation trajectory. For example, on Planar, \citet{qin2025defog} use polydec-2, while \citet{luosimgfm} use $g_k(t)=1-(1-t)^k$ with $k=10$. Our hazard-guided clock instead directly adapts sampling times to predicted graph activity. To isolate its contribution, we compare uniform, polynomial, and hazard-guided schedules under the \ours{} sampler on Planar; see \cref{tab:planar_clock_ablation}. Our hazard-guided clock provides the best trade-off.

\begin{wraptable}{r}{0.35\textwidth}
\vspace{0\baselineskip}
\centering
\setlength{\tabcolsep}{2pt}

\caption{Design choices.}
\label{tab:sampling_ablation}
\vspace{-0.5em}

{\scriptsize
\resizebox{\linewidth}{!}{%
\begin{tabular}{lcc|cc}
\toprule
& \multicolumn{2}{c}{$T=50$}
& \multicolumn{2}{c}{$T=100$} \\
\cmidrule(lr){2-3}\cmidrule(lr){4-5}
Variant
& Val. $\uparrow$ & FCD $\downarrow$
& Val. $\uparrow$ & FCD $\downarrow$ \\
\midrule
\textbf{\ours{}}
& 89.5 & 0.98 & 89.9 & 0.96 \\
\midrule
Hazard-only
& 88.1 & 1.10 & 89.2 & 1.12 \\
w/o early stopping
& 88.1 & 1.41 & 88.7 & 1.82 \\
w/o hazard annealing
& 88.5 & 1.43 & 87.9 & 1.39 \\
\bottomrule
\end{tabular}%
}
\par}
\end{wraptable} 
\paragraph{Design choices.}
\sloppy In our experiments, we use confidence-weighted hazard selection (\cref{eq:confidence-weighted-hazard}), early stopping, and time annealing to the hazard-guided clock.
We ablate our design choices in \cref{tab:sampling_ablation}. 
First, we find that confidence-weighted hazard selection can improve quality over hazard-only selection. Second, we find that removing early stopping reduces sample quality, suggesting that further edits at low predicted activity can introduce errors. 
Also, we find that directly applying the time policy in \cref{eq:hazard-target-time} leads to large time jumps early in the generation process. To moderate this, we interpolate between the current time $t_k$ and the hazard-guided time $\widehat t_k$: \(t_{k+1}=(1-\gamma_k)t_k+\gamma_k\widehat t_k\), where \(\gamma_k\in(0,1]\) limits the time jumps in earlier steps, allowing the model to perform more graph refinements; see \cref{app:additional_experiments} for more details.

\begin{wraptable}{r}{0.35\textwidth}
\vspace{0\baselineskip}
\centering
\setlength{\tabcolsep}{2pt}

\caption{Sampler comparison.}
\label{tab:hglow_sampler_comparison}
\vspace{-0.5em}

{\scriptsize
\resizebox{\linewidth}{!}{%
\begin{tabular}{lcc|cc}
\toprule
& \multicolumn{2}{c}{$T=50$}
& \multicolumn{2}{c}{$T=100$} \\
\cmidrule(lr){2-3}\cmidrule(lr){4-5}
Sampler
& Val. $\uparrow$ & FCD $\downarrow$
& Val. $\uparrow$ & FCD $\downarrow$ \\
\midrule
\textbf{\ours{}}
& 89.5 & 0.98 & 89.9 & 0.96 \\
\midrule
Euler
& 80.3 & 1.38 & 85.8 & 1.38 \\
RVF
& 86.4 & 1.03 & 87.7 & 1.08 \\
\bottomrule
\end{tabular}%
}
\par}
\end{wraptable}
\textbf{Sampler comparison.}
Using the same \ours{} model trained on MOSES with the decomposed hazard-jump parametrization, we compare our \ours{} sampler with the parallel Euler sampler and the RVF variant proposed in \citet{luosimgfm}; see \cref{tab:hglow_sampler_comparison}. 
Compared to other samplers, our \ours{} sampler achieves higher validity and lower FCD at $T=50$ than Euler and RVF at $T=100$, demonstrating more effective use of the sampling budget.
In \cref{app:additional_experiments}, we provide an ablation comparing denoiser and hazard–jump parameterizations under different sampling strategies.

\vspace{-0,5em}
\section{Related work}\label{sec:related_work}
\vspace{-0,5em}
\textbf{Discrete diffusion and flow matching for graphs.}
Discrete-state graph diffusion models rely on categorical corruption toward simple priors, such as uniform or empirical marginal distributions over node and edge attributes 
\citep{austin2021structured,vignac2022digress}, with recent extensions introducing re-noising and learned critics to identify corrupted nodes and edges \citep{boget2025simple} to improve sample quality. 
Recent graph diffusion approaches build on the continuous-time formulations based on CTMCs, which allow flexible sampling-time discretizations \citep{campbell2022continuous,xu2024discrete,siraudin2024cometh}; however, their reverse dynamics remain tied to the forward noising process.
In contrast, discrete flow matching (DFM) \citep{campbell2024generative,gat2024discrete} decouples training from the choice of sampling dynamics. Recent graph adaptations \citep{qin2025defog,luosimgfm} leverage the training-sampling decoupling and improve performance by optimizing rate matrices and time distortions during inference. While diffusion- and flow-based graph models typically learn a denoiser network, \citet{li2026neural} recently proposed a diffusion model that directly parameterizes the network to predict hazards and conditional jumps. We adapt this parameterization within the generator-matching framework \citep{holderrieth2025generator} and propose a structured generator-learning framework for graphs.

\textbf{Improved sampling.}
Continuous-time approaches adopt finite-time approximators to simulate the CTMC trajectory, including independent Euler updates and $\tau$-leaping \citep{campbell2022continuous,gat2024discrete}, where additional predictor-corrector steps can be used for further refinement. 
\citet{liu2025think} propose decomposing the generation into planning and denoising, and introduce a latent variable that indicates whether a token is corrupted. In contrast, \ours{} decomposes the generative process by directly exploiting the decomposition of CTMC generators into hazards and conditional jumps to guide sampling without an auxiliary planner network. In masked diffusion language models, iterative refinement is achieved by remasking previously decoded tokens, thereby enabling additional computation to improve sample quality \citet{wang2026remasking}.
\vspace{-0,6em}
\section{Conclusion}
We present \ours{}, a hazard-guided framework for accelerating graph generation. By factorizing generators into hazards and conditional jump distributions, \ours{} leverages the event structure of CTMC generators to coordinate graph edits. Under the same decomposition, we propose a novel DFM model for graph generation. We limit our experiments to accelerate pretrained DFM models; however, \ours{} can be readily applied to discrete diffusion models with accessible generators, which we leave for future work. 
Also, \ours{} benefits from the empirical prior over node and edge attributes, where initialized edges are dominated by the no-bond category, unlocking accelerated molecular generation within a few NFEs. To accommodate graphs that require extensive updates, we provide an extension of \ours{} that enables parallel updates per step. Future work could explore selection strategies that exploit graph structure to coordinate parallel edits more effectively.

\subsection*{AI use statement}
In this work, we used generative AI tools for language editing, including grammar, spelling, and word choice, as well as for assisting in polishing mathematical proofs and formulating ideas. We have not used generative AI tools to generate synthetic datasets, develop theoretical models or conceptual frameworks, provide critical ingredients for proving mathematical claims,  propose or refine hypotheses, design or provide feedback on research methodology or experiments, implement methods, assist with translation, clean or reformat datasets, support qualitative or thematic data analysis, or interpret results. We have reviewed all AI-assisted work. We take responsibility for the final content of this work, including text, claims, or artifacts produced with the aid of generative AI.

\subsection*{Ethics statement}
In this work, we propose a generative model for discrete data and apply it to the generation of molecules and synthetic graphs. As generative models can be applied in various fields, we see potential positive and negative societal impacts. Generative modeling can accelerate scientific discovery; however, there are potential societal consequences, specifically those related to the generation of unsafe drugs.

\subsection*{Reproducibility statement}
The experimental setup is described in \cref{app:exp_details} and \cref{app:app_sampling}.

\newpage

\bibliography{iclr2027_conference}
\bibliographystyle{iclr2027_conference}

\newpage
\appendix
\crefalias{section}{appendix}
\crefname{appendix}{Appendix}{Appendices}
\Crefname{appendix}{Appendix}{Appendices}

\raggedbottom

\section{Proofs}\label{app:objective}
\begingroup
\renewcommand{\thetheorem}{1}
\begin{theorem}[Decomposed graph generator matching objective]
Assume that the conditional target and learned graph
generators follow the single-component construction in
\cref{prop:componentwise_graph_generator} and admit hazard-jump decompositions. Then, the conditional generator matching objective decomposes as:
\begin{equation}
\mathcal{L}(\theta)
=
\mathbb{E}_{\substack{
G_1\sim p_1,\;t\sim\mathcal{T},\\
G_t\sim p_{t\mid 1}(\cdot\mid G_1)
}}
\left[
\sum_{n=1}^{N}\mathcal{L}_{\theta,t}^{X,(n)}
+
\sum_{1\le i<j\le N}\mathcal{L}_{\theta,t}^{E,(ij)}
\right],
\end{equation}
where, for each graph component $c$, 
\begin{equation}
\mathcal{L}_{\theta,t}^{(c)}
=
D_{\mathrm{Poi}}\!\left(
\hazard_t^{(c)}
\,\middle\|\,
\hazard_{\theta,t}^{(c)}
\right)
+
\hazard_t^{(c)}
D_{\mathrm{KL}}\!\left(
\boldsymbol{\jump}_t^{(c)}
\,\middle\|\,
\boldsymbol{\jump}_{\theta,t}^{(c)}
\right).
\label{eq:app-componentwise-decomposed-loss}
\end{equation}
Here, $D_{\mathrm{Poi}}(a\|b)=a\log\frac{a}{b}-a+b$.
\end{theorem}
\endgroup

\begin{proof}
We write the conditional generator matching objective in its generalized KL form:
\begin{align}
\mathcal{L}(\theta)
=
\mathbb{E}\Bigg[
\sum_{\widetilde G\neq G_t}
\Bigg(
&
R_t(G_t,\widetilde G\mid G_1)
\log
\frac{R_t(G_t,\widetilde G\mid G_1)}
     {R_{\theta,t}(G_t,\widetilde G)}
\nonumber\\
&-
R_t(G_t,\widetilde G\mid G_1)
+
R_{\theta,t}(G_t,\widetilde G)
\Bigg)
\Bigg],
\end{align}
where the expectation is over
$G_1\sim p_1$, $t\sim\mathcal{T}$, and
$G_t\sim p_{t\mid 1}(\cdot\mid G_1)$.

Under the construction in
\cref{prop:componentwise_graph_generator}, every transition
with a nonzero off-diagonal rate changes exactly one
component. Thus, the sum over graph transitions separates
into sums over node and edge transitions.

Fix a graph component $c$. For readability, we define
\begin{align}
r_t^{(c)}(\tilde z)
&=
R_t^{(c)}
\left(z_t^{(c)},\tilde z\mid G_1\right),
\\
r_{\theta,t}^{(c)}(\tilde z)
&=
R_{\theta,t}^{(c)}
\left(z_t^{(c)},\tilde z;G_t\right),
\end{align}
and suppress the conditioning arguments of the hazards
and jump distributions. Therefore,
\begin{equation}
r_t^{(c)}(\tilde z)
=
\hazard_t^{(c)}\jump_t^{(c)}(\tilde z),
\qquad
r_{\theta,t}^{(c)}(\tilde z)
=
\hazard_{\theta,t}^{(c)}
\jump_{\theta,t}^{(c)}(\tilde z),
\quad \text{for } \tilde z\neq z_t^{(c)}.
\end{equation}

The contribution of component $c$ to the objective is
\begin{equation}
\mathcal{L}_{\theta,t}^{(c)}
= \sum_{\tilde z\neq z_t^{(c)}}
\left[
r_t^{(c)}(\tilde z)
\log
\frac{r_t^{(c)}(\tilde z)}
     {r_{\theta,t}^{(c)}(\tilde z)}-r_t^{(c)}(\tilde z)+r_{\theta,t}^{(c)}(\tilde z)\right].
\end{equation}
First suppose that $\hazard_t^{(c)}>0$ and that the
learned rates are positive wherever the target rates
are positive. Substituting the hazard-jump
decompositions and using
\begin{equation}
\sum_{\tilde z\neq z_t^{(c)}}
\jump_t^{(c)}(\tilde z)
=
\sum_{\tilde z\neq z_t^{(c)}}
\jump_{\theta,t}^{(c)}(\tilde z)
=
1,
\end{equation}
we obtain
\begin{align}
\mathcal{L}_{\theta,t}^{(c)}
={}&
\hazard_t^{(c)}
\log
\frac{\hazard_t^{(c)}}
     {\hazard_{\theta,t}^{(c)}}
-
\hazard_t^{(c)}
+
\hazard_{\theta,t}^{(c)}
\nonumber+
\hazard_t^{(c)}
\sum_{\tilde z\neq z_t^{(c)}}
\jump_t^{(c)}(\tilde z)
\log
\frac{\jump_t^{(c)}(\tilde z)}
     {\jump_{\theta,t}^{(c)}(\tilde z)}
\\
={}&
D_{\mathrm{Poi}}\!\left(
\hazard_t^{(c)}
\,\middle\|\,
\hazard_{\theta,t}^{(c)}
\right)+
\hazard_t^{(c)}D_{\mathrm{KL}}\!\left(\boldsymbol{\jump}_t^{(c)}\,\middle\|\,
\boldsymbol{\jump}_{\theta,t}^{(c)}\right).
\end{align}
Terms with zero target mass are interpreted using the
standard convention $0\log(0/q)=0$.
If a learned rate is zero where the target rate is
positive, both sides are $+\infty$, so the identity also holds in this case.
When $\hazard_t^{(c)}=0$, all target rates vanish and
\begin{equation}
\mathcal{L}_{\theta,t}^{(c)}
= \sum_{\tilde z\neq z_t^{(c)}} r_{\theta,t}^{(c)}(\tilde z)
=\hazard_{\theta,t}^{(c)}
=D_{\mathrm{Poi}}\!\left(0\,\middle\|\,\hazard_{\theta,t}^{(c)}\right).
\end{equation}
The target jump distribution may then be chosen arbitrarily over alternative attributes, and its hazard-weighted KL term is defined as zero.
Components with no alternative attributes have zero hazard and contribute zero.

Finally, summing over node and edge components gives
\begin{equation}
\mathcal{L}(\theta)
=
\mathbb{E}_{\substack{
G_1\sim p_1,\;t\sim\mathcal{T},\\
G_t\sim p_{t\mid 1}(\cdot\mid G_1)
}}
\left[
\sum_{n=1}^{N}\mathcal{L}_{\theta,t}^{X,(n)}
+
\sum_{1\le i<j\le N}\mathcal{L}_{\theta,t}^{E,(ij)}
\right],
\end{equation}
which proves the result.
\end{proof}

\begingroup
\renewcommand{\theproposition}{2}
\begin{proposition}[Graph hazard]
For the generator in
\cref{prop:componentwise_graph_generator},
the total graph hazard is given by the sum of the componentwise node and edge hazards:
\begin{equation}
H_{\theta,t}\left(G_t\right)
=
\sum_{n}\hazard_{\theta,t}^{X,(n)}\left(x_t^{(n)}; G_t\right)
+
\sum_{i<j}\hazard_{\theta,t}^{E,(ij)}\left(e_t^{(ij)}; G_t\right).
\label{eq:total-graph-exit-rate}
\end{equation}
\end{proposition}
\endgroup

\begin{proof}
By definition, the total graph hazard is the sum of the transition
rates to all graphs different from the current graph:
\begin{equation}
H_{\theta,t}(G_t)
=
\sum_{G'\neq G_t}
R_{\theta,t}(G_t,G').
\end{equation}
Since the graph generator permits only single-component transitions,
every graph $G'\neq G_t$ with a nonzero transition rate differs from
$G_t$ in exactly one node or one edge attribute. Hence, the sum can be
partitioned according to the component that changes:
\begin{align}
H_{\theta,t}(G_t)
&=
\sum_{n}
\sum_{\tilde{x}\neq x_t^{(n)}}
R_{\theta,t}^{X,(n)}
\left(x_t^{(n)},\tilde{x}; G_t\right)
\nonumber\\
&\quad+
\sum_{i<j}
\sum_{\tilde{e}\neq e_t^{(ij)}}
R_{\theta,t}^{E,(ij)}
\left(e_t^{(ij)},\tilde{e}; G_t\right).
\end{align}
The componentwise hazards are defined as the corresponding sums of
outgoing rates:
\begin{align}
\hazard_{\theta,t}^{X,(n)}\left(x_t^{(n)}; G_t\right)
&=
\sum_{\tilde{x}\neq x_t^{(n)}}
R_{\theta,t}^{X,(n)}
\left(x_t^{(n)},\tilde{x}; G_t\right),\\
\hazard_{\theta,t}^{E,(ij)}\left(e_t^{(ij)}; G_t\right)
&=
\sum_{\tilde{e}\neq e_t^{(ij)}}
R_{\theta,t}^{E,(ij)}
\left(e_t^{(ij)},\tilde{e}; G_t\right).
\end{align}
Substituting these definitions yields
\eqref{eq:total-graph-exit-rate}.
\end{proof}

\section{Experimental details}\label{app:exp_details}
\begin{table}[t]
\centering
\caption{Dataset statistics. $E$ and $X$ denote the numbers of present-edge
and node categories, respectively. The no-edge category is excluded from
$E$. $B$ counts directed edge entries, following the convention of prior
work, while $N$ denotes the number of nodes.}
\label{tab:dataset_statistics}
\vspace{-0.5em}
\resizebox{\linewidth}{!}{%
\begin{tabular}{lrc*{6}{r}}
\toprule
Dataset
& \# Graphs
& Train/Val/Test
& $E$
& $X$
& Avg.\ $B$
& Max.\ $B$
& Avg.\ $N$
& Max.\ $N$ \\
\midrule
QM9
& 130,831
& 97,734/20,042/13,055
& 4 & 4
& 18.9 & 28
& 8.8 & 9 \\

MOSES
& 1,936,962
& 1,584,663/176,225/176,074
& 4 & 8
& 46.5 & 62
& 21.7 & 27 \\

GuacaMol
& 1,398,213
& 1,118,633/69,926/209,654
& 4 & 12
& 60.4 & 176
& 27.8 & 88 \\

Planar
& 200
& 128/32/40
& 1 & 1
& 355.7 & 362
& 64.0 & 64 \\

Tree
& 200
& 128/32/40
& 1 & 1
& 126.0 & 126
& 64.0 & 64 \\
\bottomrule
\end{tabular}%
}
\end{table}
\subsection{Datasets and Metrics}\label{app:datasets} 
We evaluate \ours{} on the molecular graph benchmarks QM9, MOSES, and
GuacaMol, following the preprocessing and dataset splits used in prior
discrete graph generative models
\citep{vignac2022digress,siraudin2024cometh,qin2025defog,xu2024discrete}. 
QM9 license is CC BY 4.0, MOSES and GuacaMol datasets are licensed under the MIT License. 
Molecules are represented as attributed undirected graphs, where nodes
correspond to atom types and edges encode bond types, with the absence of a
bond treated as an additional edge category. For each generated graph, the number of nodes is sampled from the empirical graph-size distribution of the
training set. We summarize all data statistics in \cref{tab:dataset_statistics}.

\paragraph{QM9.}
We evaluate \ours{} on QM9 \citep{ramakrishnan2014quantum}. Following prior work \citep{vignac2022digress}, we kekulize the molecules using the RDKit library and remove the hydrogen atoms. We use the same data split as \citet{vignac2022digress} and train on 100K molecules. For the evaluation, we sample 10K molecules and report four metrics. The \textbf{Validity} is measured by sanitizing the sampled molecules and converting them to SMILES strings using the RDKit library. We also report the \textbf{Relaxed Validity}, which accounts for charged molecules. \textbf{Uniqueness} is measured across valid molecules, and it captures the proportion of non-isomorphic graphs within the generated samples. We also report the \textbf{Fréchet ChemNet Distance (FCD)} metric, which embeds the generated samples and the test data using the ChemNet neural network and compares the resulting distributions using the Wasserstein-2 distance.

\paragraph{MOSES.}
We evaluate \ours{} on the MOSES \citep{polykovskiy2020molecular} benchmark. 
We sample 25K samples and report the {Validity}, as well as the Uniqueness. We further report the \textbf{Novelty}, which measures how many of these samples are non-isomorphic to any graph in the training set. We further include the \textbf{Filters} metric, which indicates the percentage of generated samples that pass the filters applied when constructing the
dataset. In addition, we report \textbf{Scaffold similarity (Scaf)}, which compares the frequencies of the Bemis-Murcko scaffolds in the generated set and a reference set, as well as the \textbf{Similarity to the nearest neighbor (SNN)}, which computes the average Tanimoto similarity between the molecular fingerprints of a generated set and the fingerprints of the molecules of a
reference set

\paragraph{Guacamol.}
We evaluate \ours{} on the Guacamol \citep{brown2019guacamol} benchmark. 
We note that, as Guacamol includes molecules with charges, the generated graphs are converted to charged molecules based on the relaxed validity criterion used by \citet{jo2022score} before being translated to the corresponding SMILES strings.

We sample 10K samples following prior baselines and report the {Validity}, as well as the percentage of \textbf{Valid and Unique (V.U.)} samples, and \textbf{Valid, Unique, and Novel (V.U.N.)} samples. The GuacaMol benchmark computes two scores, the \textbf{FCD score} and the \textbf{KL divergences (KL)} between the distributions of a set of physicochemical descriptors in the training set and the generated set.

\paragraph{Synthetic graphs.}
We evaluate the generation performance of \ours{} using the synthetic benchmarks Planar \citep{martinkus2022spectre} and Tree \citep{bergmeister2024efficient}. Planar graphs are connected planar graphs that can be drawn on a plane without any edges crossing, while tree graphs are connected acyclic graphs. 

We follow prior work \citep{vignac2022digress,qin2025defog} and report V.U.N, which is the percentage of valid, unique, and novel graphs, as well as the  Ratio metric.
The ratio is computed by first calculating the MMD distances between the training and test sets for the same graph statistics. The final Ratio metric is then computed by dividing the average MMD distance between the generated sets and test sets by the average MMD distance between the training and test sets.

\subsection{Baselines.}
We compare against representative discrete graph diffusion and flow-matching
methods, including DiGress~\citep{vignac2022digress},
DISCO~\citep{xu2024discrete}, COMETH~\citep{siraudin2024cometh}, and
DeFoG~\citep{qin2025defog}. We additionally use other iterative denoising-based methods whenever comparable results are available. We note that we re-evaluate the released checkpoints from \citet{qin2025defog}, where the performance is generally consistent with the paper’s results, with minor variations attributed to training and sampling stochasticity.

\subsection{Architecture and training settings}\label{app:training_settings}

\paragraph{Graph-transformer backbone.}
Following DiGress~\citep{vignac2022digress}, we use a graph transformer that jointly processes node, edge, and graph-level features. Given a noisy graph $G_t$ and time $t$, the network produces representations for every node and unordered node pair. We retain the multi-head attention backbone and introduce
a decomposed output parameterization with separate predictions for hazard and conditional jumps. We augment the input with Relative Random Walk Probabilities (RRWP) \citep{vignac2022digress,qin2025defog}.

\textbf{Parametrization.}
\ours{} directly parameterizes the componentwise generators
through separate hazard and jump prediction heads:
\[
f_\theta(G_t, t)
=
\left(
    \boldsymbol{\hazard}_{\theta,t}^{X},
    \boldsymbol{\jump}_{\theta,t}^{X},
    \boldsymbol{\hazard}_{\theta,t}^{E},
    \boldsymbol{\jump}_{\theta,t}^{E}
\right),
\]
where $\boldsymbol{\hazard}_{\theta,t}^{X}$ and $\boldsymbol{\hazard}_{\theta,t}^{E}$ are non-negative hazards, and $\boldsymbol{\jump}_{\theta,t}^{X}$ and $\boldsymbol{\jump}_{\theta,t}^{E}$ are the predicted jump distributions satisfying: 
$
    \sum_{\tilde z\neq z_t^{(c)}}
    \jump_{\theta,t}^{(c)}
    (\tilde z\mid z_t^{(c)},G_t)
     =1, \text{ and }
   \jump_{\theta,t}^{(c)}
    (z_t^{(c)}\mid z_t^{(c)},G_t)
     =0. 
$
\ours{} directly parameterizes the componentwise generators through separate hazard and conditional jump predictions.
For a component $c$ with transformer representation $h_t^{(c)}$, we first compute$u_t^{(c)}=\operatorname{ReLU}\left(W_h h_t^{(c)}+b_h\right).$
The network produces logits:
$
s^{(c)} = w_a^\top u_t^{(c)} + b_a, \ell^{(c)} = W_r u_t^{(c)} + b_r.
$
We parameterize the hazard as
\[
\tau_{\theta,t}^{(c)} = a_t\,\sigma\!\left(s^{(c)}\right),
\qquad
a_t = \frac{\dot{\alpha}_t}{1-\alpha_t},
\]
where $\sigma$ denotes the sigmoid function, ensuring
$\tau_{\theta,t}^{(c)} \in [0,a_t]$.

We note that for the jump head, we mask the
current category $i=z_t^{(c)}$ and apply softmax over the
remaining categories:
\[
\jump{}_{\theta,t}^{(c)}(j)
=
\begin{cases}
\displaystyle
\frac{\exp(\ell_{j}^{(c)})}
{\sum_{k\neq i}\exp(\ell_{k}^{(c)})},
& j\neq i,\\[6pt]
0, & j=i.
\end{cases}
\]
We use separate node and edge output MLPs. Each MLP shares its hidden features between the hazard and jump projections. Edge logits are symmetrized across node pairs.


\paragraph{Settings.}
The shared componentwise heads, permutation-compatible
structural features, and symmetric edge outputs preserve the permutation equivariance of the graph-transformer backbone. We provide an overview of the training hyperparameters of \ours{} in \cref{tab:training_settings}. We note that we use as prior distribution $p_0$ the empirical marginal prior over node and edge attributes \citep{vignac2022digress}. We further use the train time distortion setting from \citet{qin2025defog}. We note that for sampling, we use \ours{}'s inferred time discretization, as opposed to the sample time distortions used in \citet{qin2025defog}, see \cref{app:app_sampling}.
\begin{table}[t]
\centering
\caption{Overview of the training hyperparameters of \ours{}.}
\label{tab:training_settings}
\resizebox{\linewidth}{!}{%
\begin{tabular}{lccccc}
\hline
& QM9 & GuacaMol & MOSES & Planar & Tree \\
\hline
Epochs
& 1000 & 250 & 300 & 100000 & 100000 \\
Learning rate
& $2\times10^{-4}$ & $2\times10^{-4}$ & $2\times10^{-4}$
& $2\times10^{-4}$ & $2\times10^{-4}$ \\
Optimizer
& AdamW & AdamW & AdamW & AdamW & AdamW \\
AMSGrad
& Yes & Yes & Yes & Yes & Yes \\
RRWP steps
& 12 & 20 & 20 & 32 & 128 \\
Number of layers
& 9 & 12 & 12 & 10 & 10 \\
Number of heads
& 8 & 8 & 8 & 8 & 8 \\
Prior distribution
& Marginal & Marginal & Marginal & Marginal & Marginal \\
Train Distortion
& Identity & Polydec & Polydec & Identity & Polydec-2 \\
Sample Distortion
& Adaptive (Ours) & Adaptive (Ours) & Adaptive (Ours)
& Adaptive (Ours) & Adaptive (Ours) \\
\hline
\end{tabular}%
}
\end{table}

\section{Sampling details}\label{app:app_sampling}
\subsection{Derivation of the Hazard-based time allocation}
\label{app:hazard-time}

For readability, we initially suppress the dependence of componentwise generators on $\theta$, $t$, and $G_t$. For each component $c$, we follow the uniformization
construction in \citep{rao2013fast} and define:
\begin{equation}
\Lambda^{(c)}
=
\sum_{\tilde z}
\max_{z\neq\tilde z} R^{(c)}(z,\tilde z).
\end{equation}
This sums the largest incoming off-diagonal rate for each
destination and bounds every row's exit hazard:
\begin{equation}
\sum_{\tilde z\neq z} R^{(c)}(z,\tilde z)
\leq \Lambda^{(c)},
\quad \forall z.
\end{equation}
For a component with only one possible attribute value, we
set $\Lambda^{(c)}=0$. For $\Lambda^{(c)}>0$, we define
\begin{equation}
\mathbf{P}^{(c)}
=
\mathbf{I}
+
\frac{\mathbf{R}^{(c)}}{\Lambda^{(c)}}.
\end{equation}
The bound above ensures that all entries of
$\mathbf{P}^{(c)}$ are nonnegative. Its rows sum to one
because the rows of $\mathbf{R}^{(c)}$ sum to zero.
Thus, $\mathbf{P}^{(c)}$ is a stochastic matrix, with
\begin{equation}
P^{(c)}(z_t^{(c)},z_t^{(c)})
=
1-\frac{\hazard^{(c)}}{\Lambda^{(c)}},
\qquad
\hazard^{(c)}
=
-R^{(c)}(z_t^{(c)},z_t^{(c)}).
\end{equation}
This diagonal entry is the probability of retaining the
current attribute at one uniformization event for the fixed
component matrix. If $\Lambda^{(c)}=0$, then
$\mathbf{R}^{(c)}=\mathbf{0}$, and we set
$\mathbf{P}^{(c)}=\mathbf{I}$ and define the normalized
hazard to be zero.

As a sampling policy, we choose the target time by assigning
equal weight to the mean node and edge retention probabilities:
\begin{equation}\label{eq:time_infered_from_rate}
\widehat t
=
1-\frac{1}{2}
\left(
\frac{1}{N}\sum_{i=1}^{N}
\frac{\hazard_{\theta,t}^{(i)}}{\Lambda_{\theta,t}^{(i)}}
+
\frac{1}{\binom{N}{2}}\sum_{1\leq i<j\leq N}
\frac{\hazard_{\theta,t}^{(ij)}}{\Lambda_{\theta,t}^{(ij)}}
\right).
\end{equation}
Since every normalized hazard lies in $[0,1]$, this definition
yields $\widehat t\in[0,1]$. 

\paragraph{Denoiser-based DFM generators.}
We first instantiate the policy under the base DFM generator
construction \citep{gat2024discrete}. Let
$a_t=\dot\alpha_t/(1-\alpha_t)$, with $\alpha_t<1$ and
$\dot\alpha_t>0$. We evaluate the denoiser at $G_t$ and
hold its predicted distribution fixed when constructing
the component matrix:
\begin{equation}
R_{\theta,t}^{(c)}(z,\tilde z;G_t)
=
a_t
\left[
p_{1\mid t}^{\theta,(c)}(\tilde z\mid G_t)
-
\delta(z,\tilde z)
\right].
\end{equation}
For $z\neq\tilde z$, these entries are independent of the
dummy source attribute $z$. Thus, for a component with at
least two possible attribute values,
\begin{equation}
\begin{aligned}
\Lambda_{\theta,t}^{(c)}
&=
\sum_{\tilde z}
\max_{z\neq\tilde z}
R_{\theta,t}^{(c)}(z,\tilde z;G_t)
\\
&=
a_t\sum_{\tilde z}
p_{1\mid t}^{\theta,(c)}(\tilde z\mid G_t)
=
a_t.
\end{aligned}
\end{equation}
Components with only one possible attribute value have
zero hazard and contribute zero. Substituting into
\cref{eq:time_infered_from_rate} gives
\begin{equation}
\widehat t
=
1-\frac{1-\alpha_t}{2 \dot \alpha_t}
\left(
\frac{1}{N}\sum_{i=1}^{N}\hazard_{\theta,t}^{X,(i)}
+
\frac{1}{\binom{N}{2}}
\sum_{1\leq i<j\leq N}\hazard_{\theta,t}^{E,(ij)}
\right).
\label{eq:app-denoiser-time}
\end{equation}
Here, $\hazard_{\theta,t}^{(c)}/a_t\in[0,1]$, so
$\widehat t\in[0,1]$ without clipping.


\subsection{Additional sampling details}
In our experiments, we set the early stopping threshold to 1. Stopping when the predicted graph activity falls below this threshold improves generation quality compared with continuing to edit the graph until the full sampling budget $T$ is exhausted. 



\section{Additional experiments}\label{app:additional_experiments}

\begin{wraptable}{r}{0.4\textwidth}
\centering
\setlength{\tabcolsep}{1.5pt}
\renewcommand{\arraystretch}{1.1}
\vspace{0\baselineskip}
\caption{Parameterization comparison.}
\label{tab:qm9_euler}
\vspace{-1em}
{\scriptsize
\resizebox{\linewidth}{!}{%
\begin{tabular}{llcc|cc|cc}
\toprule
& & \multicolumn{2}{c}{$T=20$}
& \multicolumn{2}{c}{$T=50$}
& \multicolumn{2}{c}{$T=100$} \\
\cmidrule(lr){3-4}\cmidrule(lr){5-6}
\cmidrule(lr){7-8}
Model & Sampler
& Val. $\uparrow$ & FCD $\downarrow$
& Val. $\uparrow$ & FCD $\downarrow$
& Val. $\uparrow$ & FCD $\downarrow$ \\
\midrule
Denoiser & Euler
& {97.0} & 0.75
& 98.9 & 0.28
& 99.4 & 0.14 \\
\textbf{\ours{}} & Euler
& 96.7 & {0.66}
& {99.2} & {0.24}
& {99.5} & {0.13} \\
\midrule
Denoiser & \ours{}
& 98.9 & {0.31}
& 99.0 & 0.21
& 99.0 & {0.11} \\
\textbf{\ours{}} & \ours{}
& {99.2} & 0.37
& {99.1} & {0.10}
& {99.2} & {0.10} \\
\bottomrule
\end{tabular}%
}
\par}
\end{wraptable}
\subsection{Parameterization.}
In the following, we conduct a controlled experiment on QM9 to compare the denoiser parametrization trained with cross-entropy against the hazard–jump parametrization trained with the decomposed objective in \cref{eq:componentwise-decomposed-loss-full}. Both models use identical data splits, a graph-transformer backbone, and matched initialization and optimization. We sample 10K samples using Euler and \ours{}. 
\ours{} sampling substantially improves the validity–FCD trade-off over Euler sampling, particularly at small sampling budgets. Combined with our decomposed parameterization, \ours{} sampling delivers consistently strong performance across the evaluated sampling budgets.

\begin{table}
\centering
\vspace{-1\baselineskip}
\caption{NFE across datasets and sampling steps.}
\label{tab:app_NFE}
\begin{tabular}{@{}lcr@{\,${}\pm{}$\,}l@{}}
\toprule
Dataset & Steps $T$ & \multicolumn{2}{c}{NFE} \\
\midrule
QM9      & $20$  & 19.9 & 0.4 \\
QM9      & $50$  & 22.4 & 2.3 \\
\midrule
MOSES    & $50$  & 35.9 & 7.0 \\
MOSES    & $100$ & 47.1 & 15.7 \\
MOSES    & $200$ & 57.2 & 14.2 \\
\midrule
GuacaMol & $50$  & 40.4 & 8.1 \\
GuacaMol & $100$ & 49.2 & 13.7 \\
GuacaMol & $200$ & 78.8 & 35.6 \\
\bottomrule
\end{tabular}
\end{table}

\subsection{Clock comparisons}
In \cref{fig:time_distortions}, we compare different time distortions, including time distortions from prior work, such as polydec-2 and polydec-10 \citep{qin2025defog,luosimgfm}, as well as our hazard-guided inferred time. We omit the early stopping and hazard annealing here.

\begin{figure}[hptb]
    \centering
    \includegraphics[width=0.4\linewidth]
        {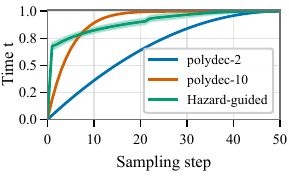}
    \caption{Schedulers.}
    \label{fig:time_distortions}
\end{figure}

\subsection{Early stopping}
As discussed in \cref{sec:sampling}, \ours{} performs early stopping during inference to avoid disrupting node or edge updates that can render the graph invalid. We report the average number of function evaluations (NFE) in \cref{tab:app_NFE} for different benchmarks and budgets $T$.

\subsection{Time annealing}
To moderate large changes in the inferred time, we interpolate between the current model query time \(t_k\) and the hazard-guided target \(\widehat t_k\):
\[
t_{k+1}=(1-\gamma_k)t_k+\gamma_k\widehat t_k,
\qquad \gamma_k\in(0,1],
\]
where smaller \(\gamma_k\) can limit the magnitude of each time update, allowing further graph refinement. This is particularly useful during the initial sampling steps, when changes in the graph can produce substantial fluctuations in the predicted activity. We use a piecewise-constant schedule,
\[
\gamma_k=
\begin{cases}
\gamma_0, & k<K_{\mathrm{prefix}},\\
1, & k\geq K_{\mathrm{prefix}},
\end{cases}
\]
which damps the initial time updates and subsequently adopts the inferred target directly. This interpolation preserves the possibility of backward time updates, with renoising applied whenever \(t_{k+1}<t_k\). We use $\gamma_0=0.1$. For QM9, Moses, and GuacaMol, we set $K_{\mathrm{prefix}}$ to 20$\%$ of the total steps $T$. For Planar and Tree, $K_{\mathrm{prefix}}$ is 30$\%$ of $T$.

\subsection{Sampler comparison}
We provide the full quantitative comparison of the samplers comparisons under the same trained \ours{} model in \cref{tab:hglow_sampler_comparison_full_table}.
\begin{table}
\centering
\setlength{\tabcolsep}{4pt}
\caption{Sampler comparison using the HGLow trained model on MOSES.}
\label{tab:hglow_sampler_comparison_full_table}

\resizebox{0.6\linewidth}{!}{%
\begin{tabular}{clccccc}
\toprule
$T$ & Sampler
& Validity\,$\uparrow$
& Unique\,$\uparrow$ & Novelty\,$\uparrow$ & Filters\,$\uparrow$
& FCD\,$\downarrow$ \\
\midrule

50 & RVF (default)
& 86.4
& {100.0} & 94.6 & {99.1}
& 1.03 \\

50 & DeFoG (default)
& 80.3
& {100.0} & 95.3 & 98.0
& 1.38 \\


50 & HGLow (ours)
& {89.5}
& {100.0} & 94.4 & 99.1
& {0.98} \\

\midrule

100 & RVF (default)
& 87.7
& {100.0} & 94.2 & 99.2
& 1.08 \\

100 & DeFoG (default)
& 85.8
& {100.0} & 90.7 & {99.4}
& 1.38 \\


100 & HGLow (ours)
& {89.9}
& 99.9 & 93.0 & 99.2
& {0.96} \\

\bottomrule
\end{tabular}%
}
\vspace{0\baselineskip}
\end{table}
\begin{table}
\centering
\setlength{\tabcolsep}{3pt} 
\caption{Large molecule generation performance on MOSES and GuacaMol.}
\label{tab:app_large_molecule_gen}
\resizebox{0.9\textwidth}{!}{%
\begin{tabular}{lcc
    *{4}{S[table-format=3.1]}
    *{2}{S[table-format=1.2]}
    S[table-format=2.1]
    |*{5}{S[table-format=3.1]}
}
\toprule
& & & \multicolumn{7}{c|}{MOSES}
& \multicolumn{5}{c}{GuacaMol} \\
\cmidrule(lr){4-10}\cmidrule(lr){11-15}

Model & Class & Steps
& {Val.\,$\uparrow$}
& {Unique.\,$\uparrow$}
& {Novelty\,$\uparrow$}
& {Filters\,$\uparrow$}
& {FCD\,$\downarrow$}
& {SNN\,$\uparrow$}
& {Scaf.\,$\uparrow$}
& {Val.\,$\uparrow$}
& {V.U.\,$\uparrow$}
& {V.U.N.\,$\uparrow$}
& {KLD.\,$\uparrow$}
& {FCD\,$\uparrow$} \\
\midrule

Training set & -- & --
& 100.0 & 100.0 & 0.0 & 100.0 & 0.01 & 0.64 & 99.1
& 100.0 & 100.0 & 0.0 & 99.9 & 92.8 \\

\midrule

DiGress{\tiny~\citep{vignac2022digress}}
& \small{Diff.} & 500
& 85.7 & 100.0 & 95.0 & 97.1 & 1.19 & 0.52 & 14.8
& 85.2 & 85.2 & 85.1 & 92.9 & 68.0 \\




DisCo{\tiny~\citep{xu2024discrete}}
& \small{Diff.} & 500
& 88.3 & 100.0 & 97.7 & 95.6 & 1.44 & 0.50 & 15.1
& 86.6 & 86.6 & 86.5 & 92.6 & 59.7 \\




Cometh{\tiny~\citep{siraudin2024cometh}}
& \small{Diff.} & 500
& 86.1 & 100.0 & 96.0 & 97.3 & 1.27 & 0.50 & 16.4
& 94.4 & 94.4 & 93.5 & 94.1 & 67.4 \\

Cometh-PC{\tiny~\citep{siraudin2024cometh}}
& \small{Diff.} & 500
& 90.5 & 99.9 & 92.6 & 99.1 & 1.27 & 0.54 & 16.0
& 98.9 & 98.9 & 97.6 & 96.7 & 72.7 \\


\midrule

SimGFM{\tiny~\citep{luosimgfm}}
& \small{Flow} & 50
& 86.1 & 100.0 & 95.1 & 98.6 & 1.44 & 0.51 & 16.9
& 98.5 & 98.5 & 98.3 & 93.6 & 77.1 \\

SimGFM{\tiny~\citep{luosimgfm}}
& \small{Flow} & 200
& 89.4 & 100.0 & 94.6 & 98.9 & 1.08 & 0.55 & 17.3
& 99.2 & 99.2 & 98.5 & 94.6 & 76.4 \\

SimGFM{\tiny~\citep{luosimgfm}}
& \small{Flow} & 500
& 91.1 & 100.0 & 93.9 & 99.2 & 1.19 & 0.56 & 14.4
& 99.2 & 99.2 & 98.6 & 95.5 & 73.5 \\

\midrule



DeFoG{\tiny~\citep{qin2025defog}} $(\eta=0,\omega=0)$
& \small{Flow} & 50
& 88.7 & 100.0 & 95.9 & 98.0 & 3.49 & 0.53 & 11.8
& 84.2 & 84.2 & 83.8 & 91.0 & 32.7 \\

DeFoG{\tiny~\citep{qin2025defog}} $(\eta=0,\omega=0)$
& \small{Flow} & 100
& 89.5 & 100.0 & 95.6 & 98.2 & 2.51 & 0.54 & 12.6
& 92.3 & 92.3 & 91.8 & 92.9 & 51.4 \\

DeFoG{\tiny~\citep{qin2025defog}} $(\eta=0,\omega=0)$
& \small{Flow} & 200
& 89.6 & 100.0 & 94.7 & 98.6 & 1.93 & 0.55 & 12.8
& 94.6 & 94.6 & 94.1 & 93.8 & 62.7 \\

DeFoG{\tiny~\citep{qin2025defog}} $(\eta=0,\omega=0)$
& \small{Flow} & 500
& 89.4 & 100.0 & 94.9 & 98.6 & 1.66 & 0.55 & 13.3
& 96.1 & 96.1 & 95.5 & 94.1 & 71.9 \\

\midrule

DeFoG{\tiny~\citep{qin2025defog}}
& \small{Flow} & 50
& 82.9 & 99.9 & 95.1 & 97.8 & 1.79 & 0.53 & 19.8
& 91.7 & 91.7 & 91.2 & 92.3 & 57.9 \\

DeFoG{\tiny~\citep{qin2025defog}}
& \small{Flow} & 100
& 88.5 & 99.9 & 90.7 & 99.2 & 1.75 & 0.58 & 16.7
& 96.7 & 96.7 & 95.7 & 96.0 & 74.2 \\

DeFoG{\tiny~\citep{qin2025defog}}
& \small{Flow} & 200
& 90.9 & 99.9 & 89.7 & 99.3 & 1.67 & 0.58 & 11.9
& 98.0 & 98.0 & 96.9 & 98.3 & 70.1 \\

DeFoG{\tiny~\citep{qin2025defog}}
& \small{Flow} & 500
& 91.5 & 99.9 & 90.0 & 99.2 & 1.47 & 0.55 & 11.5
& 99.0 & 99.0 & 97.9 & 97.7 & 73.8 \\

\midrule

\rowcolor{gray!15}
\textbf{\ours{}-{DeFoG}(Ours)}
& \small{Flow} & 50
& 86.2 & 100.0 & 92.7 & 97.4 & 1.84 & 0.56 & 13.9
& 92.5 & 92.5 & 91.2 & 92.6 & 70.4 \\

\rowcolor{gray!15}
\textbf{\ours{}-{DeFoG} (Ours)}
& \small{Flow} & 100
& 91.1 & 99.9 & 92.0 & 99.2 & 1.46 & 0.57 & 16.7
& 97.9 & 97.9 & 96.3 & 97.0 & 72.4 \\

\rowcolor{gray!15}
\textbf{\ours{}-{DeFoG} (Ours)}
& \small{Flow} & 200
& 91.1 & 100.0 & 93.0 & 99.2 & 1.32 & 0.56 & 14.6
& 99.1 & 99.1 & 97.5 & 97.3 & 74.1 \\

\rowcolor{gray!15}
\textbf{\ours{}-{DeFoG} (Ours)}
& \small{Flow} & 500
& 91.0 & 100.0 & 92.0 & 99.3 & 1.37 & 0.57 & 15.7
& 99.3 & 99.3 & 98.0 & 98.0 & 74.2 \\

\midrule

\rowcolor{gray!15}
\textbf{\ours{} (Ours)}
& \small{Flow} & 50
& 89.5 & 100.0 & 94.4 & 99.1 & 0.98 & 0.55 & 18.8
& 95.2 & 95.2 & 93.6 & 89.9 & 69.3 \\

\rowcolor{gray!15}
\textbf{\ours{} (Ours)}
& \small{Flow} & 100
& 89.9 & 99.9 & 93.0 & 99.2 & 0.96 & 0.56 & 16.0
& 99.2 & 99.2 & 98.1 & 96.9 & 77.7 \\

\rowcolor{gray!15}
\textbf{\ours{} (Ours)}
& \small{Flow} & 200
& 91.9 & 99.9 & 93.1 & 99.4 & 1.08 & 0.58 & 15.5
& 99.4 & 99.4 & 97.9 & 96.8 & 77.3 
\\


\bottomrule
\end{tabular}%
}
\vspace{-1em}
\end{table}

\begin{table}
\centering
\setlength{\tabcolsep}{3pt} 
\vspace{-1\baselineskip}
\caption{ Performance on the synthetic datasets Planar and Tree. AR denotes Autoregressive approaches.
}
\vspace{-1em}
\label{tab:app_synthetic_graph_gen}

\sisetup{
    table-format = 3.1,
    table-space-text-post = {$\stdsmall{0.0}$},
    table-align-text-post = false
}

\resizebox{0.7\linewidth}{!}{%
\begin{tabular}{@{}lcc*{2}{S}|*{2}{S}@{}}
\toprule
& & & \multicolumn{2}{c}{Planar}
& \multicolumn{2}{c}{Tree}
\\
\cmidrule(lr){4-5}\cmidrule(lr){6-7}

Model & Class & Steps
& {V.U.N.\,$\uparrow$} & {Ratio\,$\downarrow$}
& {V.U.N.\,$\uparrow$} & {Ratio\,$\downarrow$}
\\
\midrule

Training set
& {--} & {--}
& 100.0 & 1.0
& 100.0 & 1.0
\\

\midrule

GraphRNN\tiny~\citep{you2018graphrnn}
& \small{AR} & {--}
& 0.0 & 490.2
& 0.0 & 607.0
\\

GRAN\tiny~\citep{liao2019efficient}
& \small{AR} & {--}
& 0.0 & 2.0
& 0.0 & 607.0
\\

BiGG\tiny~\citep{dai2020scalable}
& \small{AR} & {--}
& 5.0 & 16.0
& 75.0 & 5.2
\\

GraphGen\tiny~\citep{goyal2020graphgen}
& \small{AR} & {--}
& 7.5 & 210.3
& 95.0 & 33.2
\\

AUTOGRAPH\tiny~\citep{chen2026flatten}
& \small{AR} & {--}
& 87.5 & 1.5
& {--} & {--}
\\

EDGE\tiny~\citep{chen2023efficient}
& \small{Diff.} & 1000
& 0.0 & 431.4
& 0.0 & 850.7
\\


DiGress{\tiny~\citep{vignac2022digress}}
& \small{Diff.} & 1000
& 77.5 & 5.1
& 90.0 & 1.6
\\

HSpectre\tiny~\citep{bergmeister2024efficient}
& \small{Diff.} & {--}
& 95.0 & 2.1
& 100.0 & 4.0
\\

GruM\tiny~\citep{jo2023graph}
& \small{Diff.} & {--}
& 90.0 & 1.8
& {--} & {--}
\\

DisCo{\tiny~\citep{xu2024discrete}}
& \small{Diff.} & 500
& 83.6 & {--}
& {--} & {--}
\\

Cometh{\tiny~\citep{siraudin2024cometh}}
& \small{Diff.} & 500
& 92.5 & {--}
& {--} & {--}
\\

Cometh-PC{\tiny~\citep{siraudin2024cometh}}
& \small{Diff.} & {--}
& 99.5 & {--}
& {--} & {--}
\\

CatFlow\tiny~\citep{eijkelboom2024variational}
& \small{Flow} & {--}
& 80.0 & {--}
& {--} & {--}
\\

\midrule

SimGFM{\tiny~\citep{luosimgfm}}
& \small{Flow} & 20
& 94.0 {$\stdsmall{4.4}$} & 2.3 {$\stdsmall{0.6}$}
& 88.0 {$\stdsmall{4.8}$} & 2.5 {$\stdsmall{0.9}$}
\\

SimGFM{\tiny~\citep{luosimgfm}}
& \small{Flow} & 50
& 99.5 {$\stdsmall{1.0}$} & 1.8 {$\stdsmall{0.5}$}
& 97.0 {$\stdsmall{1.0}$} & 2.0 {$\stdsmall{0.7}$}
\\

SimGFM{\tiny~\citep{luosimgfm}}
& \small{Flow} & 200
& 100.0 {$\stdsmall{0.0}$} & 9.3 {$\stdsmall{2.6}$}
& 99.5 {$\stdsmall{1.0}$}
& 1.5 {$\stdsmall{0.2}$}
\\

SimGFM{\tiny~\citep{luosimgfm}}
& \small{Flow} & 1000
& 100.0 {$\stdsmall{0.0}$} & 21.0 {$\stdsmall{1.5}$}
& 100.0 {$\stdsmall{0.0}$} & 2.1 {$\stdsmall{0.8}$}
\\

\midrule

DeFoG{\tiny~\citep{qin2025defog}}
& \small{Flow} & 50
& 95.0 {$\stdsmall{3.2}$} & 3.2 {$\stdsmall{1.1}$}
& 73.5 {$\stdsmall{9.0}$} & 2.5 {$\stdsmall{1.0}$}
\\

DeFoG{\tiny~\citep{qin2025defog}}
& \small{Flow} & 1000
& 99.5 {$\stdsmall{1.0}$}
& 1.6 {$\stdsmall{0.4}$}
& 96.5 {$\stdsmall{2.6}$}
& 1.6 {$\stdsmall{0.4}$}
\\

\midrule

\rowcolor{gray!15}
\textbf{\ours{} (Ours)}
& \small{Flow} & 20
& 95.0 {$\stdsmall{3.5}$} & 1.9 {$\stdsmall{0.5}$}
& 93.5 {$\stdsmall{4.6}$} & 2.2 {$\stdsmall{0.8}$}
\\

\rowcolor{gray!15}
\textbf{\ours{} (Ours)}
& \small{Flow} & 50
& 99.0 {$\stdsmall{1.2}$} & 1.9 {$\stdsmall{0.4}$}
& 98.0 {$\stdsmall{2.9}$} & 2.7 {$\stdsmall{0.9}$}
\\

\rowcolor{gray!15}
\textbf{\ours{} (Ours)}
& \small{Flow} & 100
& 99.0 {$\stdsmall{1.2}$} & 1.8 {$\stdsmall{0.3}$}
& 98.5 {$\stdsmall{2.0}$} & 1.8 {$\stdsmall{0.4}$}
\\

\rowcolor{gray!15}
\textbf{\ours{} (Ours)}
& \small{Flow} & 200
& 100.0 {$\stdsmall{0.0}$} & 2.4 {$\stdsmall{0.2}$}
& 100.0 {$\stdsmall{0.0}$} & 2.0 {$\stdsmall{0.8}$}
\\

\bottomrule
\end{tabular}%
}
\end{table}

\subsection{Accelerating pretrained models}
In \cref{tab:app_large_molecule_gen}, we provide the full quantitative comparison of DeFoG (base), without added stochasticity and target guidance during inference, DeFoG, using the sampling hyperparameters as reported in \citet{qin2025defog}, \ours{}-DeFoG, which applies the \ours{} hazard-guided sampler to accelerate the pretrained base DeFoG model. We further include our \ours{} DFM model.

\begin{algorithm}[htpb]
\caption{\ours{}: Parallel hazard-guided sampling}
\label{alg:hazard-parallel}
\begin{algorithmic}[1]
\State \textbf{Input:} model $f_\theta$, budget $T$,
maximum parallel updates $B$, threshold $T_{\mathrm{stop}}>0$ 
\State Initialize $G\sim p_0$, $t\gets0$, $k\gets0$
\While{$k<T$}
    \State Construct rates $\boldsymbol{\mathrm{R}}_{\theta,t}^{(c)}$ from $(f_\theta,G,t)$ and obtain
        $\{(\hazard_{\theta,t}^{(c)},\boldsymbol{\jump}_{\theta,t}^{(c)})\}_{c\in\mathcal C}$

    \State $H_{\theta,t}\gets\sum_{c\in\mathcal C}\hazard_{\theta,t}^{(c)}$
    \State \textbf{if} $H<T_{\mathrm{stop}}$
        \textbf{then} \Return $G$

    \State $\mathcal C_+\gets
        \{c\in\mathcal C:\hazard_{\theta,t}^{(c)}>0\}$,
        $m\gets\min(B,|\mathcal C_+|)$
    \State Select $S\subseteq\mathcal C_+$, $|S|=m$,
        without replacement with weights
        $\{\hazard_{\theta,t}^{(c)}\}_{c\in\mathcal C_+}$

    \State Sample
        $\tilde{z}^{(c)}\sim{\jump}_{\theta,t}^{(c)}(\cdot \mid z^{(c); G})$ for all $c\in S$
    \State $G\gets\operatorname{Update}
        (G,S,\{\tilde{z}^{(c)}\}_{c\in S})$

    \State $
    t\gets
        1-\frac{1-\alpha_t}{2\dot{\alpha}_t}\left(
            \overline{\hazard}_{\theta,t}^{X}
            +
            \overline{\hazard}_{\theta,t}^{E}
            \right)
    $
    \State $k\gets k+1$
\EndWhile
\State \Return $G$
\end{algorithmic}
\end{algorithm}
\subsection{Synthetic datasets}
We summarize the parallel \ours{} sampler in \cref{alg:hazard-parallel}. For planar and tree, we set $B$ to 1024 $T=20$, and use $B=512$ for $T=50$, and $B=128$ for $T=100$ and $T=200$.
We compare \ours{} to autoregressive, diffusion, and flow-based approaches in \cref{tab:app_synthetic_graph_gen}.

\section{Qualitative results}\label{app:qualitative_results}
\begin{figure}[H]
    \centering

    \begin{subfigure}{\textwidth}
        \centering
        \includegraphics[width=\textwidth]{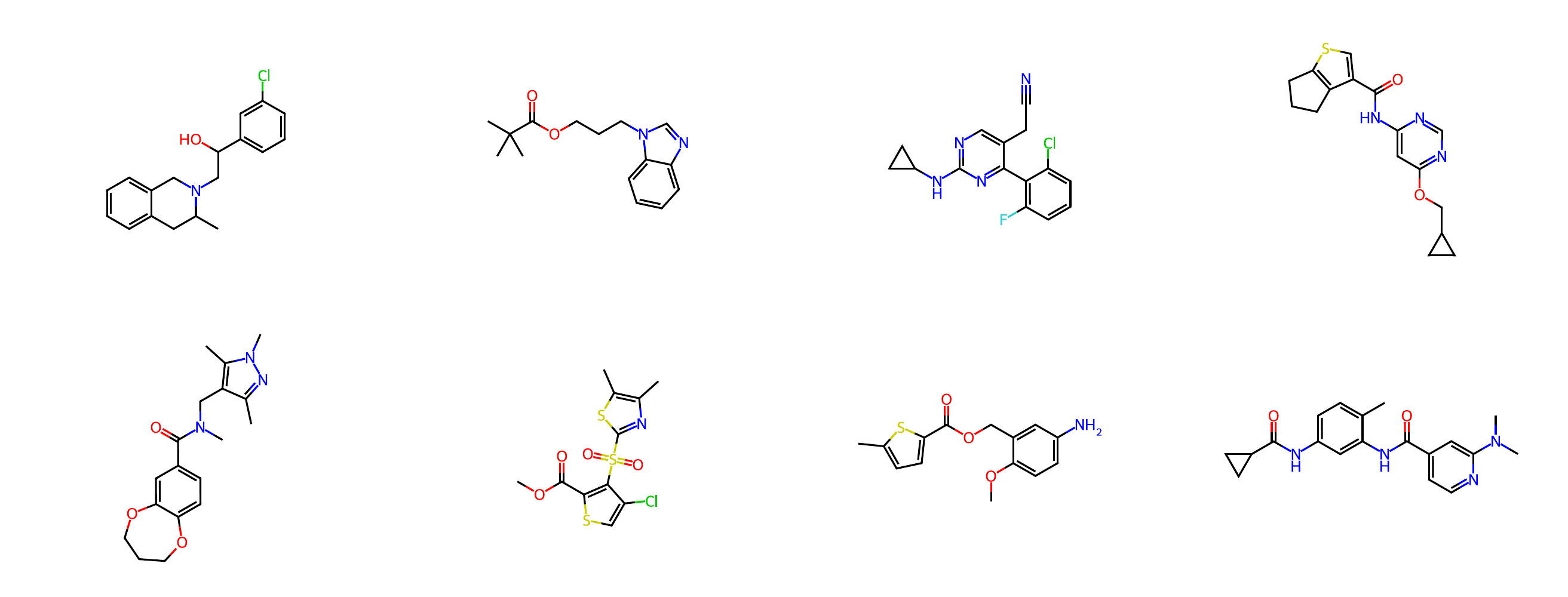}
        \caption{MOSES, $T=50$.}
    \end{subfigure}

    \vspace{0.8em}

    \begin{subfigure}{\textwidth}
        \centering
        \includegraphics[width=\textwidth]{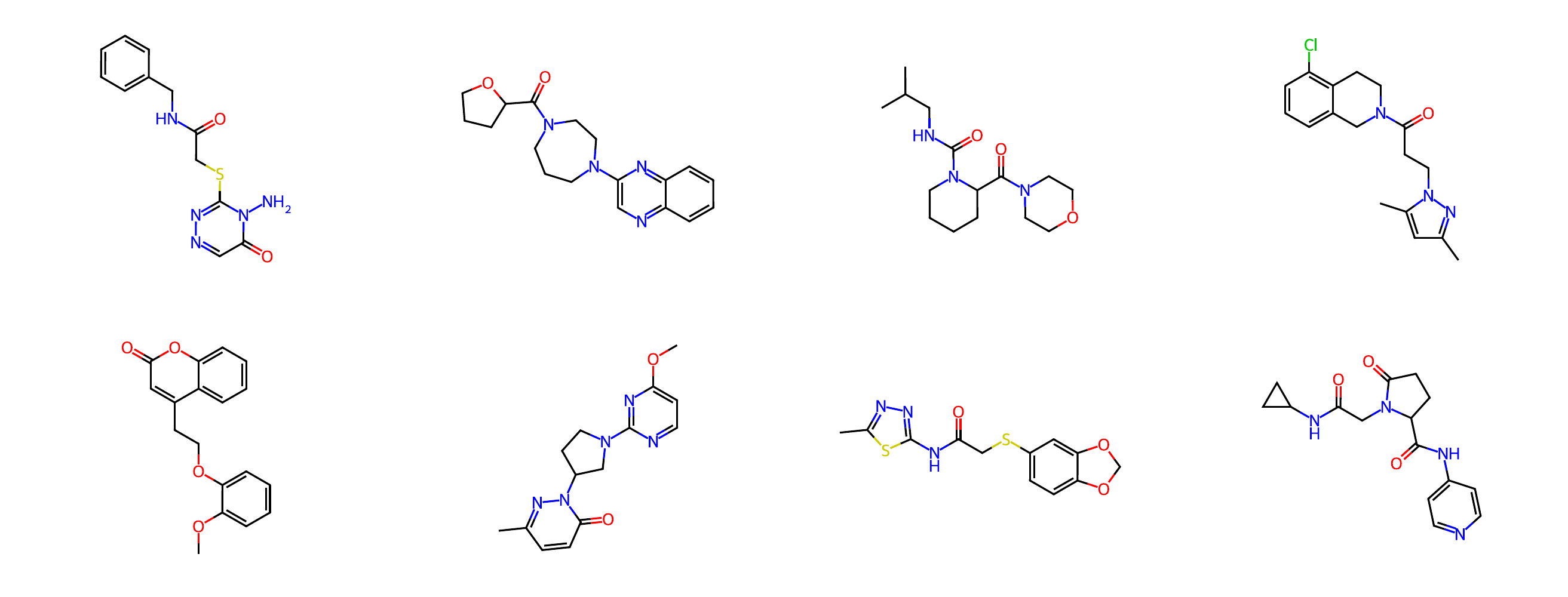}
        \caption{MOSES, $T=100$.}
    \end{subfigure}

    \vspace{0.8em}

    \begin{subfigure}{\textwidth}
        \centering
        \includegraphics[width=\textwidth]{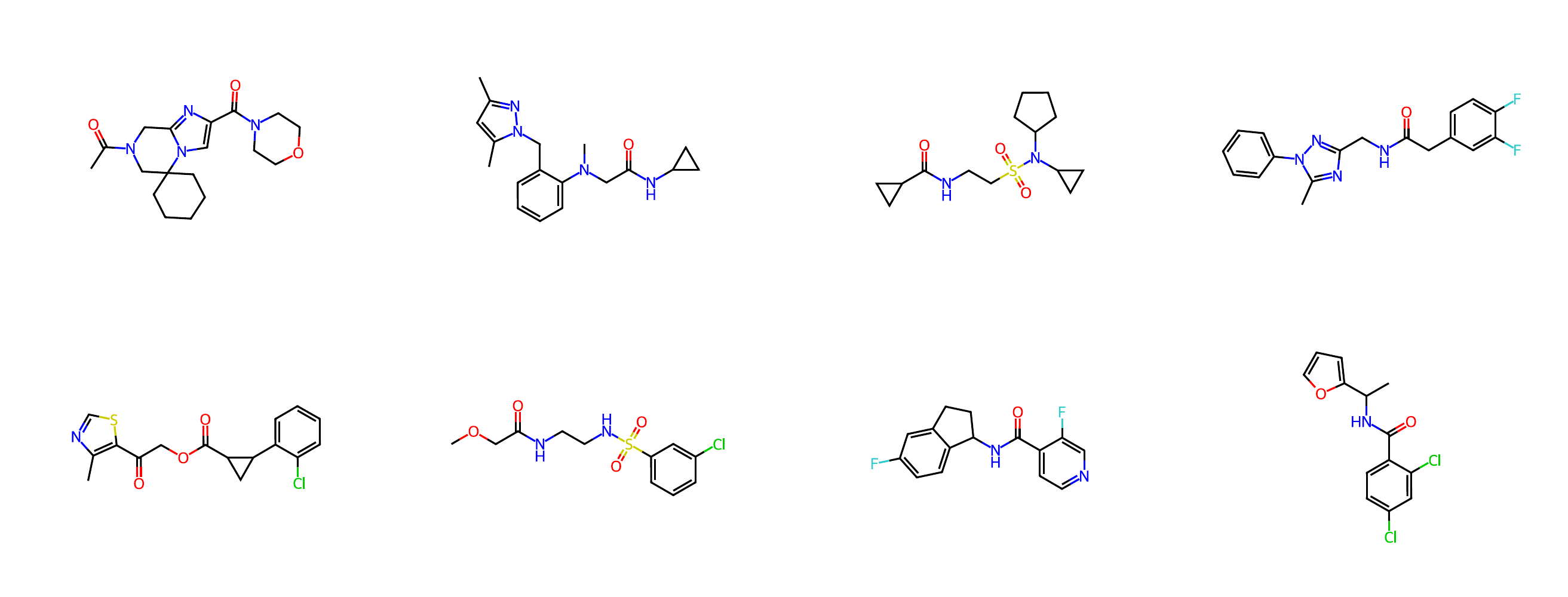}
        \caption{MOSES, $T=200$.}
    \end{subfigure}

    \caption{Uncurated sets of molecular graphs generated by \ours{} for the MOSES dataset for $T \in \{50,100,200\}$ steps.}
\end{figure}
\begin{figure}[]
    \centering

    \begin{subfigure}{\textwidth}
        \centering
        \includegraphics[width=\textwidth]{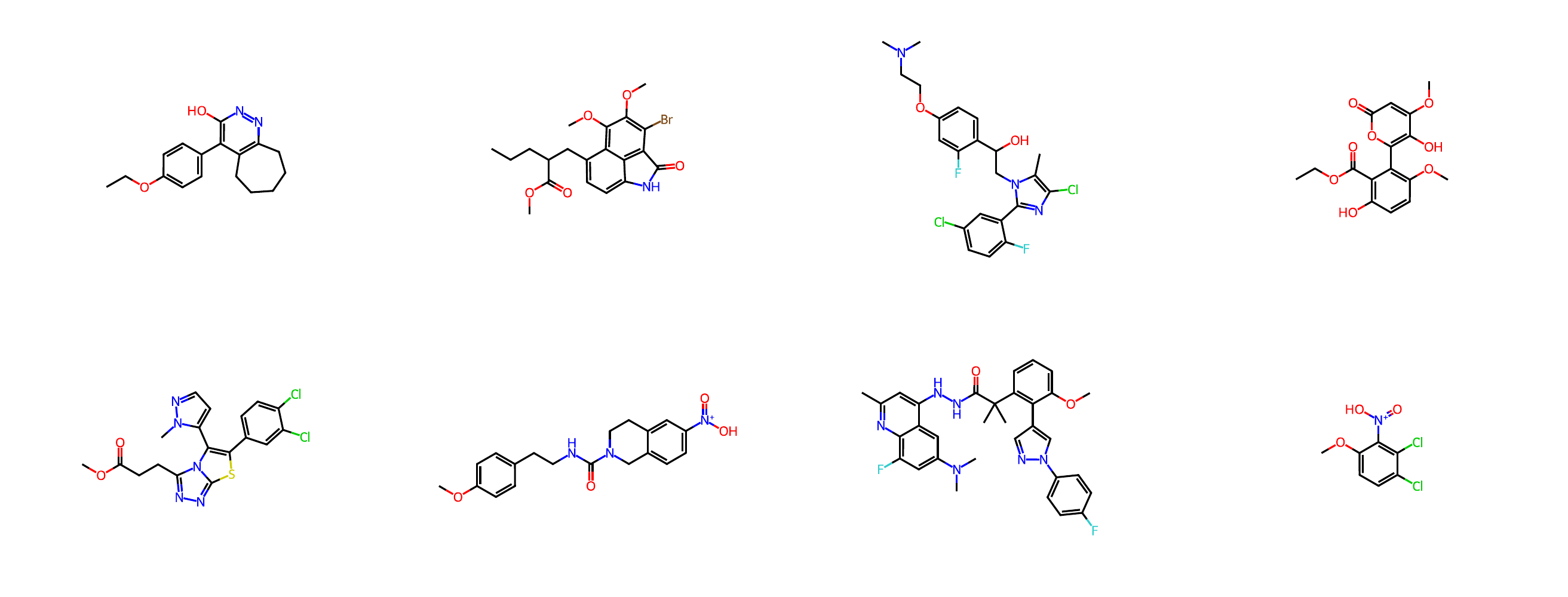}
        \caption{guacamol, $T=50$.}
    \end{subfigure}

    \vspace{0.8em}

    \begin{subfigure}{\textwidth}
        \centering
        \includegraphics[width=\textwidth]{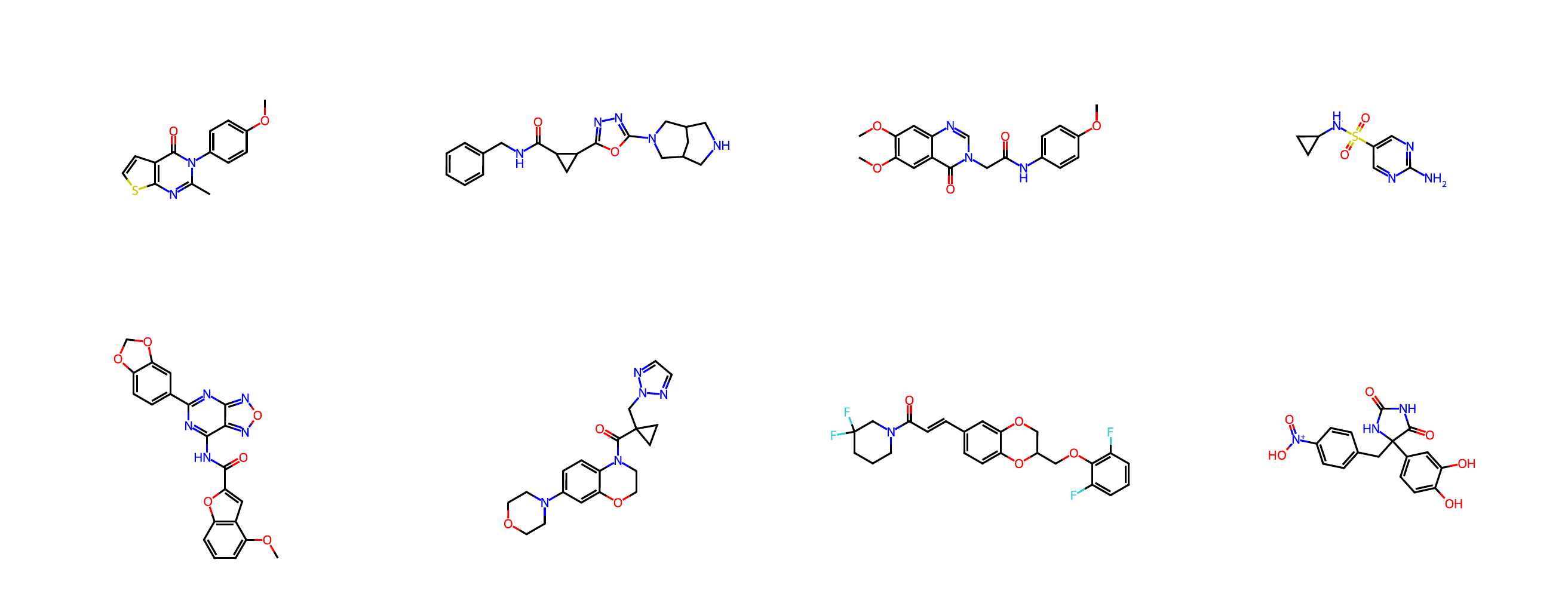}
        \caption{guacamol, $T=100$.}
    \end{subfigure}

    \vspace{0.8em}

    \begin{subfigure}{\textwidth}
        \centering
        \includegraphics[width=\textwidth]{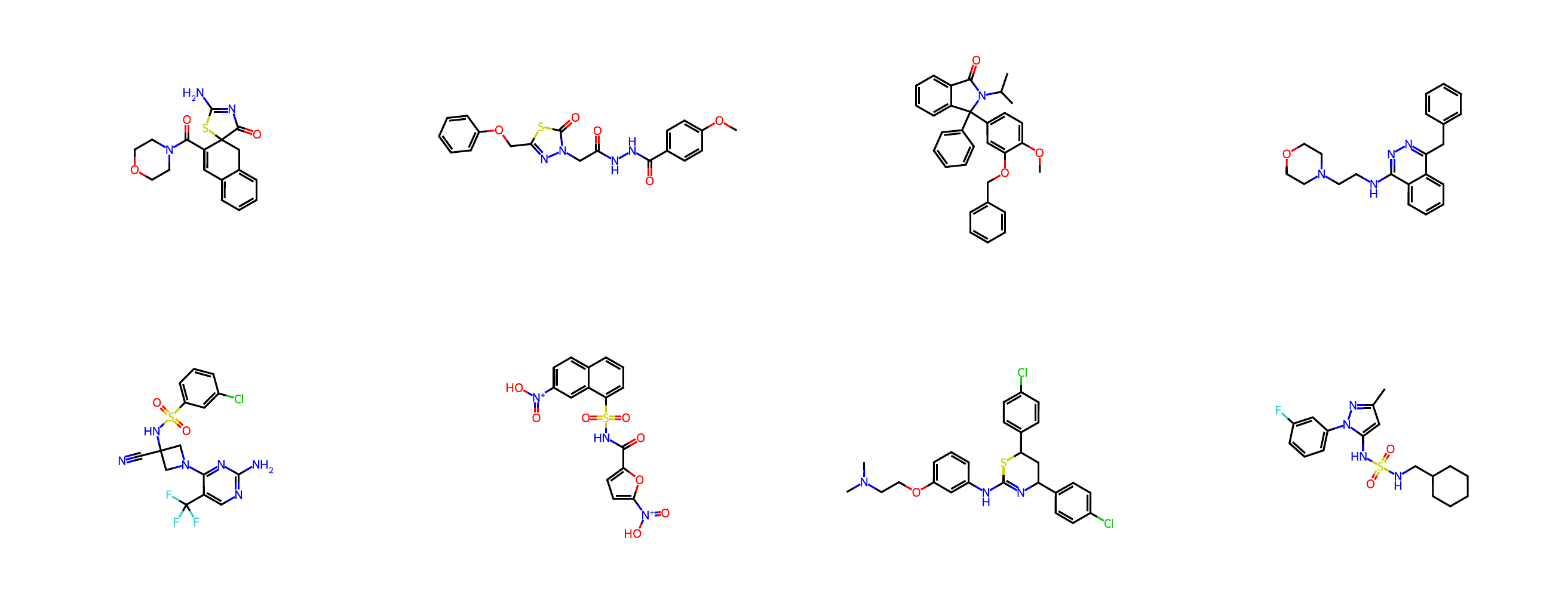}
        \caption{guacamol, $T=200$.}
    \end{subfigure}

    \caption{Uncurated sets of molecular graphs generated by \ours{} for the Guacamol dataset for $T \in \{50,100,200\}$ steps.}
    \label{fig:molecular_samples_guacamol}
\end{figure}
\begin{figure}[H]
    \centering

    \begin{subfigure}{\textwidth}
        \centering
        \includegraphics[width=\textwidth]{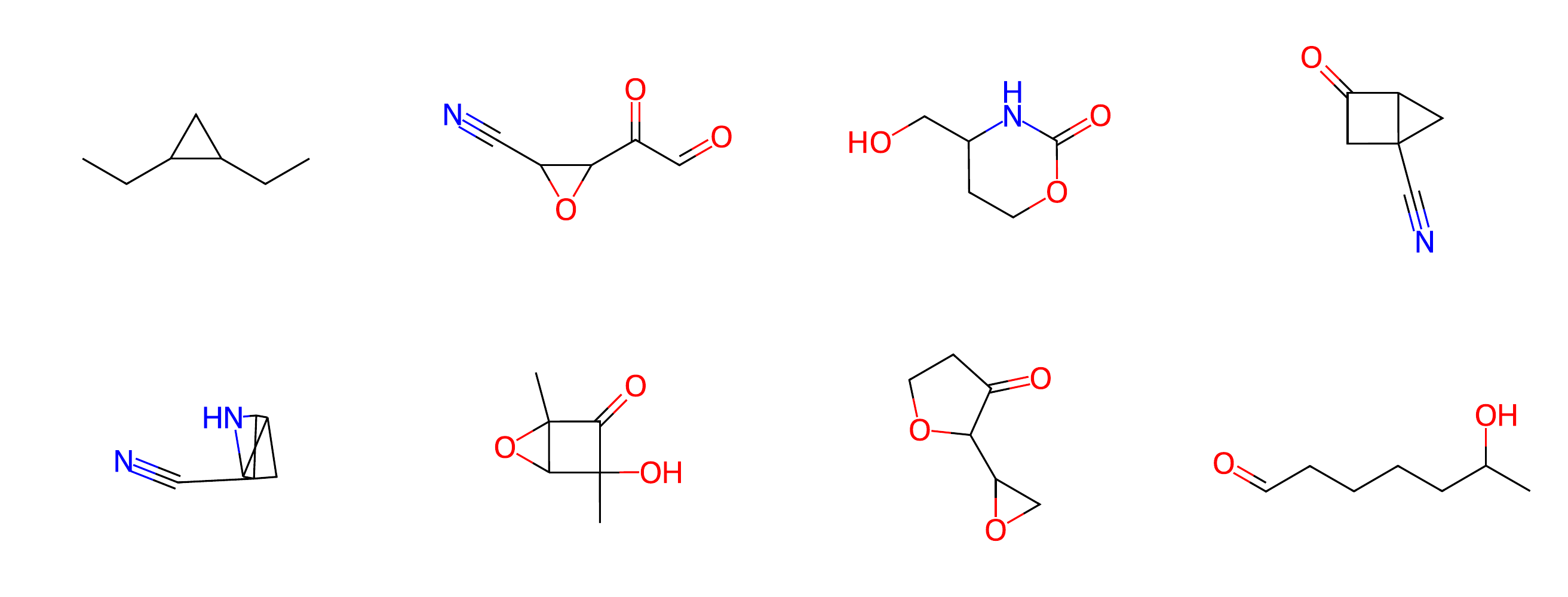}
        \caption{QM9, $T=20$.}
    \end{subfigure}

    \vspace{0.8em}

    \begin{subfigure}{\textwidth}
        \centering
        \includegraphics[width=\textwidth]{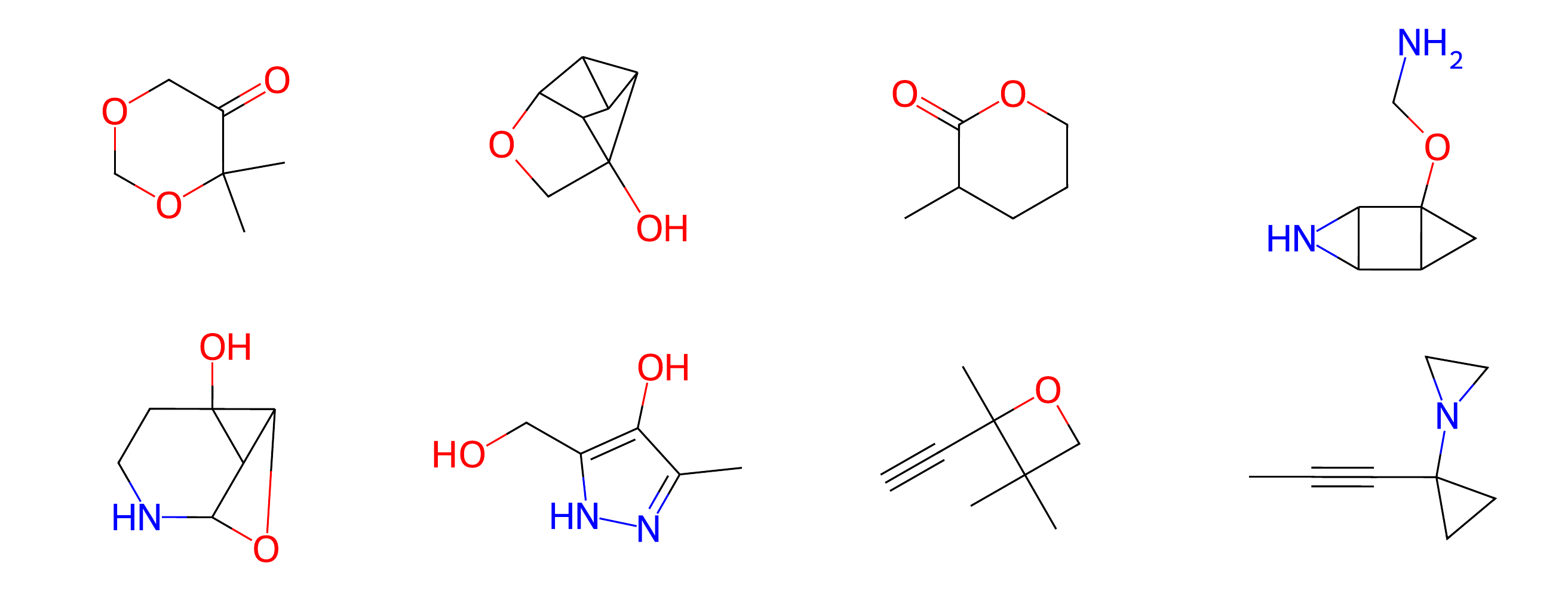}
        \caption{QM9, $T=50$.}
    \end{subfigure}

    \caption{Uncurated sets of molecular graphs generated by \ours{} for the QM9 dataset for $T \in \{20,50\}$ steps.}
\end{figure}
\begin{figure}[H]
    \centering

    \begin{subfigure}{\textwidth}
        \centering
        \includegraphics[width=0.7\textwidth]{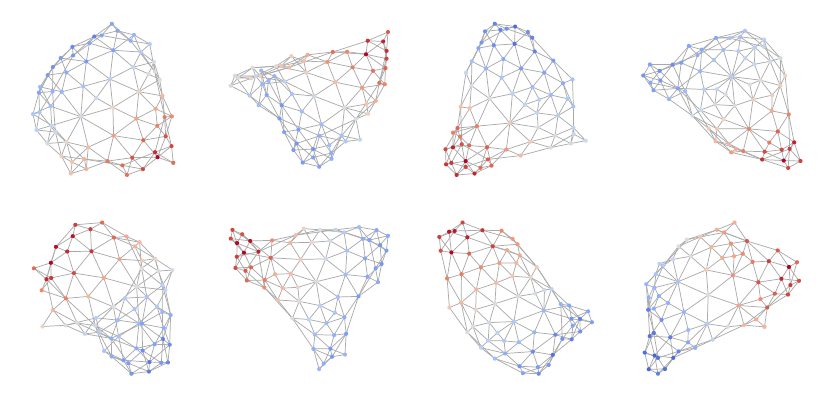}
        \caption{Planar, $T=20$.}
    \end{subfigure}

    \vspace{0.2em}

    \begin{subfigure}{\textwidth}
        \centering
        \includegraphics[width=0.7\textwidth]{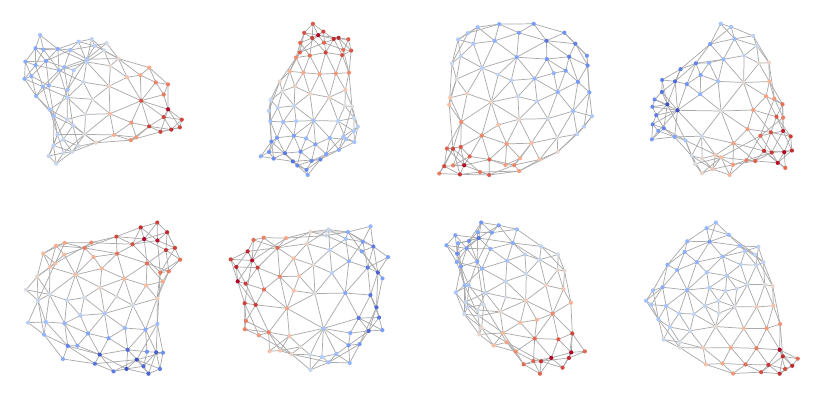}
        \caption{Planar, $T=50$.}
    \end{subfigure}

    \vspace{0.2em}

    \begin{subfigure}{\textwidth}
        \centering
        \includegraphics[width=0.7\textwidth]{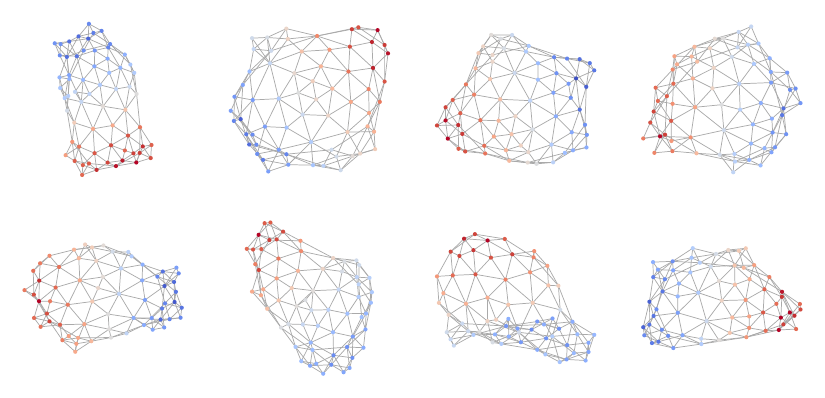}
        \caption{Planar, $T=100$.}
    \end{subfigure}

    \vspace{0.2em}

    \begin{subfigure}{\textwidth}
        \centering
        \includegraphics[width=0.7\textwidth]{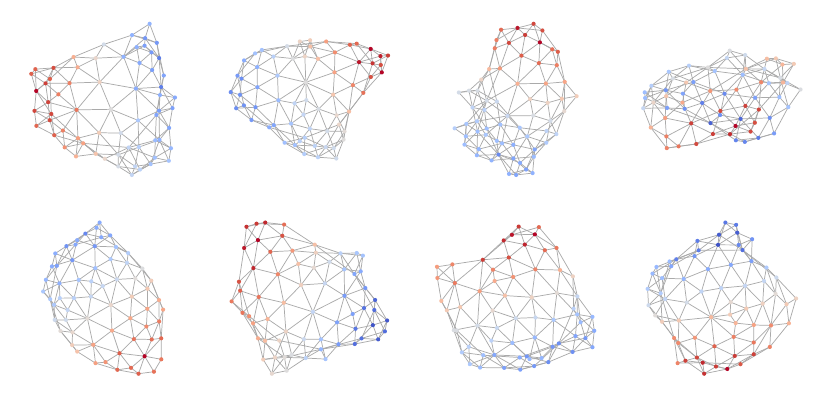}
        \caption{Planar, $T=200$.}
    \end{subfigure}

    \caption{Uncurated sets of planar graphs generated by \ours{} for the Planar dataset for $T \in \{20,50,100,200\}$ steps.}
\end{figure}

\end{document}